\documentclass[twoside,11pt]{article}

\usepackage{jmlr2e}
\usepackage{natbib}

\usepackage{multicol}
\usepackage{afterpage}
\usepackage{mathrsfs}
\usepackage{appendix}
\usepackage{hyperref}
\usepackage{amsmath}
\usepackage{amssymb}
\usepackage{graphicx}
\usepackage{amsfonts}
\usepackage{latexsym}
\usepackage{pdfsync}
\usepackage{subfigure}
\usepackage{color}
\usepackage{tabularx}
\usepackage[english]{babel}
\usepackage{array}
\usepackage{algorithm}
\usepackage{algorithmic}

\usepackage{hyperref}
 \usepackage{multirow}
 \usepackage{arydshln}
\usepackage{wrapfig}

\usepackage{enumerate}\usepackage{enumitem}

\usepackage{mathtools}

\usepackage[dvipsnames]{xcolor}

\usepackage{makecell}

\DeclareFontFamily{OT1}{pzc}{}
\DeclareFontShape{OT1}{pzc}{m}{it}{<-> s * [1.10] pzcmi7t}{}
\DeclareMathAlphabet{\mathpzc}{OT1}{pzc}{m}{it}

\newcommand{\commentout}[1]{}

\newcommand{\nwc}{\newcommand}
\nwc{\ba}{\begin{array}}
\nwc{\bal}{\begin{align}}
\nwc{\bea}{\begin{eqnarray}}
\nwc{\beq}{\begin{eqnarray}}
\nwc{\bean}{\begin{eqnarray*}}
\nwc{\beqn}{\begin{eqnarray*}}
\nwc{\beqast}{\begin{eqnarray*}}

\nwc{\ea}{\end{array}}
\nwc{\eal}{\end{align}}
\nwc{\eea}{\end{eqnarray}}
\nwc{\eeq}{\end{eqnarray}}
\nwc{\eean}{\end{eqnarray*}}
\nwc{\eeqn}{\end{eqnarray*}}
\nwc{\eeqast}{\end{eqnarray*}}

\nwc{\ep}{\varepsilon}
\nwc{\ept}{\epsilon}

\newcommand{\NN}{\mathbb{N}}

\newcommand{\RR}{\mathbb{R}}
\newcommand{\ZZ}{\mathbb{Z}}
\newcommand{\EE}{\mathbb{E}}
\newcommand{\PP}{\mathbb{P}}

\nwc{\calO}{\mathcal{O}}
\nwc{\calZ}{\mathcal{Z}}
\nwc{\calM}{\mathcal{M}}
\nwc{\calX}{\mathcal{X}}
\nwc{\calD}{\mathcal{D}}
\nwc{\calE}{\mathcal{E}}
\nwc{\calP}{\mathcal{P}}
\nwc{\calH}{\mathcal{H}}
\nwc{\calK}{\mathcal{K}}
\nwc{\calT}{\mathcal{T}}
\nwc{\calS}{\mathcal{S}}
\nwc{\tcalT}{\tilde{\calT}}
\nwc{\calA}{\mathcal{A}}
\nwc{\calB}{\mathcal{B}}
\nwc{\calI}{\mathcal{I}}
\nwc{\calC}{\mathcal{C}}
\nwc{\AS}{\calA_s}
\nwc{\ASL}{\AS^\ell}
\nwc{\ASLINF}{\AS^{\ell,\infty}}
\nwc{\Agamma}{\calA_{\gamma}}
\nwc{\BS}{\calB_s}
\nwc{\BSL}{\BS^\ell}
\nwc{\bbI}{\mathbb{I}}
\nwc{\tgamma}{\tilde{\gamma}}

\nwc{\proj}{{\rm Proj}}

\nwc{\lognn}{\frac{\log n}{n}}
\nwc{\lognnnf}{({\log n}/{n})}

\nwc{\Lam}{\Lambda}
\nwc{\hLam}{\widehat{\Lambda}}
\nwc{\hy}{\widehat{y}}

\nwc{\lcol}{\left\|}
\nwc{\rcol}{\right\|}

\nwc{\hcalP}{\widehat{\calP}}
\nwc{\hc}{\widehat{c}}
\nwc{\hP}{\widehat{\calP}}
\nwc{\hS}{\widehat{\Sigma}}
\nwc{\hV}{\widehat{V}}
\nwc{\halpha}{\widehat{\alpha}}
\nwc{\hbeta}{\widehat{\beta}}
\nwc{\hrho}{\widehat{\rho}}
\nwc{\hcalT}{\widehat{\calT}}
\nwc{\hn}{\widehat{n}}

\DeclareMathOperator*{\argmin}{arg\,min}
\nwc{\trace}{{\rm trace}}
\nwc{\rank}{{\rm rank}}
\nwc{\intdim}{{\rm intdim}}
\nwc{\din}{d_{\rm in}}
\nwc{\jmin}{j_{\min}}
\nwc{\jmax}{j_{\max}}

\nwc{\cjk}{{c_{j,k}}}
\nwc{\hcjk}{{\hc_{j,k}}}
\nwc{\Sjk}{{\Sigma_{j,k}}}
\nwc{\hSjk}{{\hS_{j,k}}}
\nwc{\Cjk}{{C_{j,k}}}
\nwc{\hCjk}{{\widehat{C}_{j,k}}}
\nwc{\calPjk}{{\calP_{j,k}}}
\nwc{\hcalPjk}{{\widehat{\calP}_{j,k}}}
\nwc{\Vjk}{{V_{j,k}}}
\nwc{\hVjk}{{\widehat{V}_{j,k}}}
\nwc{\Deltajk}{{\Delta_{j,k}}}
\nwc{\Deltajkell}{{\Delta^\ell_{j,k}}}
\nwc{\hDeltajk}{{\widehat{\Delta}_{j,k}}}
\nwc{\hDeltajkell}{{\widehat{\Delta}^\ell_{j,k}}}
\nwc{\bDeltajkell}{{\bar{\Delta}^\ell_{j,k}}}
\nwc{\jkp}{{{j+1,k'}}}
\nwc{\chijk}{{\mathbf{1}_{j,k}}}
\nwc{\chic}{{\mathbf{1}}}

\nwc{\hcalM}{\widehat\calM}
\nwc{\Voronoi}{{\rm Voronoi}}

\nwc{\fjk}{{f_{j,k}}}
\nwc{\fjkell}{f^\ell_{j,k}}
\nwc{\gjk}{{g_{j,k}}}
\nwc{\gjkell}{{g_{j,k}^\ell}}
\nwc{\alphajk}{{\alpha_{j,k}}}
\nwc{\betajk}{{\beta_{j,k}}}
\nwc{\yjk}{{y_{j,k}}}
\nwc{\Gjk}{{G_{j,k}}}
\nwc{\rhoM}{{\rho}}
\nwc{\rhojk}{{\rho_{j,k}}}
\nwc{\njk}{{n_{j,k}}}
\nwc{\hnjk}{{\widehat{n}_{j,k}}}
\nwc{\rjk}{{r_{j,k}}}
\nwc{\hrjk}{{\widehat{r}_{j,k}}}

\nwc{\hg}{\widehat{g}}
\nwc{\hf}{\widehat{f}}
\nwc{\hfjk}{{\widehat{f}_{j,k}}}
\nwc{\hfjkell}{{\widehat{f}^\ell_{j,k}}}
\nwc{\hgjk}{{\widehat{g}_{j,k}}}
\nwc{\hgjkell}{{\widehat{g}^\ell_{j,k}}}
\nwc{\halphajk}{{\widehat{\alpha}_{j,k}}}
\nwc{\hbetajk}{{\widehat{\beta}_{j,k}}}
\nwc{\hyjk}{{\widehat{y}_{j,k}}}
\nwc{\hGjk}{{\widehat{G}_{j,k}}}

\nwc{\pijk}{\pi_{j,k}}
\nwc{\hpijk}{\widehat{\pi}_{j,k}}
\nwc{\Qjk}{Q_{j,k}}
\nwc{\hQjk}{\widehat{Q}_{j,k}}
\nwc{\Wjk}{W_{j,k}}
\nwc{\hWjk}{\widehat{W}_{j,k}}

\nwc{\lamjk}{\lambda^{j,k}}
\nwc{\la}{\lambda}
\nwc{\Lamjk}{{\Lambda^{j,k}}}
\nwc{\hLamjk}{{\widehat{\Lambda}^{j,k}}}
\nwc{\hlamjk}{{\widehat{\lambda}^{j,k}}}

\nwc{\byjk}{{\bar{y}_{jk}}}

\nwc{\red}[1]{\textcolor{red}{#1}}
\nwc{\blue}[1]{\textcolor{blue}{#1}}

\nwc{\btheta}{\boldsymbol{\theta}}

\nwc{\Child}{{\mathscr{C}}}
\nwc{\amax}{a_{\rm max}}
\nwc{\amin}{a_{\rm min}}
\nwc{\diam}{{\rm diam}}

\nwc{\fj}{f_{\Lam_j}}
\nwc{\hfj}{\widehat f_{\Lam_j}}

\nwc{\calTn}{\calT^n}
\nwc{\hcalTtaun}{\hcalT_{\tau_n}}
\nwc{\calTfeta}{\calT_{(f,\eta)}}
\nwc{\hcalTeta} {\hcalT_{\eta}}
\nwc{\calTeta}{\calT_{\eta}}
\nwc{\calTtaunb}{\calT_{\tau_n/b}}
\nwc{\calTbtaun}{\calT_{b\tau_n}}
\nwc{\calTfbtaun}{\calT_{(f,b\tau_n)}}

\nwc{\hLameta}{\hLam_{\eta}}
\nwc{\Lameta}{\Lam_{\eta}}
\nwc{\Lamfeta}{\Lam_{(f,\eta)}}
\nwc{\Lambtaun}{\Lam_{b\tau_n}}
\nwc{\Lamtaunb}{\Lam_{\tau_n/b}}
\nwc{\Lamfbtaun}{\Lam_{(f,b\tau_n)}}
\nwc{\hLamtaun}{\hLam_{\tau_n}}

\jmlrheading{}{}{}{}{}{Wenjing Liao, Mauro Maggioni and Stefano Vigogna}

\ShortHeadings{Multiscale regression}{Liao, Maggioni and Vigogna}
\firstpageno{1}

\begin{document}

\title{Multiscale regression on unknown manifolds}

\author{\name Wenjing Liao \email wliao60@gatech.edu	 \\
       \addr School of Mathematics\\
       Georgia Institute of Technology, Atlanta, GA 30313, USA
       \AND
       \name Mauro Maggioni \email mauromaggionijhu@icloud.com \\
       \addr Department of Mathematics, Department of Applied Mathematics and Statistics \\
       Department of Mathematics, Mathematical Institute of Data Science \\
       Johns Hopkins University, Baltimore, MD 21218, USA
       \AND
       \name Stefano Vigogna \email vigogna@dibirs.unige.it \\
       \addr MaLGa Center, Department of Informatics, Bioengineering, Robotics and Systems Engineering\\
       University of Genova, 16145 Genova, Italy
       }

\maketitle

\begin{abstract}
We consider the regression problem of estimating functions on $\RR^D$
but supported on a $d$-dimensional manifold $ \calM \subset \RR^D $ with $ d \ll D $.
Drawing ideas from multi-resolution analysis and nonlinear approximation,
we construct low-dimensional coordinates on $\calM$ at multiple scales,
and perform multiscale regression by local polynomial fitting.
We propose a data-driven wavelet thresholding scheme
that automatically adapts to the unknown regularity of the function,
allowing for efficient estimation of functions exhibiting nonuniform regularity at different locations and scales.
We analyze the generalization error of our method
by proving finite sample bounds in high probability
on rich classes of priors.
Our estimator attains optimal learning rates (up to logarithmic factors)
as if the function was defined on a known Euclidean domain of dimension $d$,
instead of an unknown manifold embedded in $\RR^D$.
The implemented algorithm has quasilinear complexity in the sample size,
with constants linear in $D$ and exponential in $d$.
Our work therefore establishes a new framework for regression on low-dimensional sets embedded in high dimensions,
with fast implementation and strong theoretical guarantees.

\end{abstract}


\section{Introduction}

High-dimensional data challenge classical statistical models and require new understanding of tradeoffs in accuracy and efficiency.
The seemingly quantitative fact of the increase of dimension
has qualitative consequences in both methodology and implementation,
demanding new ways to break what has been called the curse of dimensionality.
On the other hand, the presence of inherent nonuniform structure in the data
calls into question linear dimension reduction techniques,
and motivates a search for intrinsic learning models.
In this paper we explore the idea of learning and exploiting the intrinsic geometry and regularity of the data in the context of regression analysis.
Our goal is to build low-dimensional representations of high dimensional functions,
while ensuring good generalization properties and fast implementation.
In view of the complexity of the data,
we allow interesting features to change from scale to scale and from location to location.
Hence, we will develop multiscale methods,
extending classical ideas of multi-resolution analysis
beyond regular domains and to the random sample regime.


In regression,
the problem is to estimate a function from a finite set of random samples.
The minimax mean squared error (MSE) for estimating functions in the H\"older space $ \calC^s([0,1]^D)$, $s>0$, is $O ( n^{ - 2 s / (2s+D) } )$,
where $n$ is the number of samples.
The exponential dependence of the minimax rate on $D$ manifests the curse of dimensionality in statistical learning,
as $ n = O ( \ep^{ -(2s+D) / s } )$ points are generally needed to achieve accuracy $\ep$.
This rate is optimal (in the minimax sense),
unless further structural assumptions are made.
For example, if the samples concentrate near a $d$-dimensional set with $ d \ll D $,
and the function belongs to a nonuniform smoothness space $ \calB^S $, with $ S > s $,
we may hope to find estimators converging in $ O( n^{ -2S / (2S+d) } ) $.
In this quantified sense, we may break the curse of dimensionality
by adapting to the intrinsic dimension and regularity of the problem.

A possible approach to this problem is based on first performing dimension reduction, and then regression in the reduced space.
Linear dimension reduction methods include principal component analysis (PCA) \citep{Pearson_PCA,Hotelling_PCA1,Hotelling_PCA2},
\linebreak for data concentrating on a single subspace,
or subspace clustering \citep{sccIJCV09,CM:CVPR2011,GPCA-VMS-PAMI05,SSC-EV-CVPR09,LRR-LLY-ICML10},
for an union of subspaces.
Going beyond linear models, we encounter
isomap \citep{Isomap}, locally linear embedding \citep{LLE}, local tangent space alignment \citep{ZhaZha}, Laplacian eigenmaps \citep{belkin:nc}, Hessian eigenmap \citep{DG_HessianEigenmaps} and diffusion map \citep{Coifman:2005:7426}.
Besides the classical Principal Component Regression \citep{jolliffe1982note},
in \cite{lee2016spectral} diffusion map is used for nonparametric regression 
expanding the unknown function over the eigenfunctions of a kernel-based operator. It is proved that, when data lie on a $d$-dimensional manifold, the MSE converges in $O(n^{-1/{O}(d^2)})$. This rate depends only on the intrinsic dimension, but does not match the minimax rate in the Euclidean space. If infinitely many unlabeled points are sampled, so that the eigenfunctions are exactly computed,
the MSE can achieve optimal rates for Sobolev functions with smoothness parameter at least $1$. Similar results hold for regression with the Laplacian eigenmaps \citep{zhou2011error}.

Some regression methods have been shown to automatically adapt to the intrinsic dimension and perform as well as if the intrinsic domain was known.
Results in this direction have been established for local linear regression \citep{BickelLi}, $k$-nearest neighbors \citep{NIPS2011_4455}, and kernel regression \citep{NIPS2013_5103}, where optimal rates depending on the intrinsic dimension were proved for functions in $\calC^2$, $\calC^1$, and $\calC^s$ with $s\le1$, respectively.
Kernel methods such as kernel ridge regression are also known to adapt to the intrinsic dimension \citep{ye2008,steinwart2009optimal},
while suitable variants of regression trees have been proved to attain intrinsic yet suboptimal learning rates  \citep{KPOTUFE20121496}.
On the other hand, dyadic partitioning estimates with piecewise polynomial regression can cover the whole scale of spaces $\calC^s$, $s>0$  \citep{GKKW},
and be combined with wavelet thresholding techniques
to optimally adapt to broader classes of nonuniform regularity \citep{BCDDT1,BCDD2}.
However, such estimators are cursed by the ambient dimension $D$,
due to the exponential cardinality of a dyadic partition of the $D$-dimensional hypercube.

This paper aims at generalizing dyadic partitioning estimates \citep{BCDDT1,BCDD2} to predict functions supported on low-dimensional sets, with optimal performance guarantees and low computational cost.
We tie together ideas in classical statistical learning \citep{ElementsStatisticalLearning,GKKW,BookNonparametricEstimation},
multi-resolution analysis \citep{daubechies1992ten,mallat1999wavelet,coifman2005geometric},
and nonlinear approximation \citep{DJ1,DJ2,Cohen}.
Our main tool is geometric multi-resolution analysis (GMRA) \citep{CM:MGM2,MMS:NoisyDictionaryLearning,LiaoMaggioni},
which is a multiscale geometric approximation scheme for point clouds in high dimensions concentrating near low-dimensional sets.
Using GMRA we learn low-dimensional local coordinates at multiple scales,
on which we perform a multiscale regression estimate by fitting local polynomials.
Inspired by wavelet thresholding techniques \citep{Cohen,BCDDT1,BCDD2},
we then compute differences between estimators at adjacent scales,
and retain the locations where such differences are large enough.
This empirically reveals where higher resolution is required to attain a good approximation,
generating a data-driven partition which adapts to the local regularity of the function.

Our approach has several distinctive features:
(i) it is \textit{multiscale}, and is therefore well-suited for data sets containing variable structural information at different scales;
(ii) it is \textit{adaptive}, allowing the function to have localized singularities or variable regularity;
(iii) it is entirely \textit{data-driven}, that is, it does not require a priori knowledge about the regularity of the function, and rather learns it automatically from the data;
(iv) it is \textit{provable}, with strong theoretical guarantees
of optimal performance on large classes of priors;
(v) it is \textit{efficient}, having straightforward implementation,
minor parameter tuning, and low computational cost.
We will prove that, for functions supported on a $d$-dimensional manifold and belonging to a rich model class characterized by a smoothness parameter $S$,
the MSE of our estimator converges at rate $O ( (\log n/n)^{ 2S/ (2S+d) } ) $. This model class contains classical H\"older continuous functions, but further accounts for potential nonuniform regularity.
Our result shows that, up to a logarithmic factor, we attain the same optimal learning rate as if the function was defined on a known Euclidean domain of dimension $d$, instead of an unknown manifold embedded in $\RR^D$.
In particular, the rate of convergence depends on the intrinsic dimension $d$ and not on the ambient dimension $D$. In terms of computation, all the constructions above can be realized by algorithms of complexity ${O}(n \log n)$, with constants linear in the ambient dimension $D$ and exponential in the intrinsic dimension $d$.

The remainder of this paper is organized as follows.
We conclude this section by defining some general notation
and formalizing the problem setup.
In Section \ref{secgmra} we review geometric multi-resolution analysis.
In Section \ref{secreg} we introduce our multiscale regression methods and establish the performance guarantees.
We discuss the computational complexity of our algorithms in Section \ref{seccomputation}.
The proofs of our results are collected in Section \ref{sec:proofs}.

\paragraph{Notation}
$f \lesssim g$ and $f \gtrsim g$ mean that there exists a positive constant $C$, independent on any variable upon which $f$ and $g$ depend, such that $f \le C g$ and $f \ge C g$, respectively. $f \asymp g$ means that both $f \lesssim g$ and $f \gtrsim g$ hold.
The cardinality of a set $A$ is denoted by $\#A$.
For $x \in \RR^D$, $\|x\|$ denotes the Euclidean norm 
and $B_r(x)$ denotes the Euclidean ball of radius $r$ centered at $x$.
Given a subspace $V \subset \RR^D$, we denote its dimension by $\dim(V)$ and the orthogonal projection onto $V$ by $\proj_V$. Let $f,g : \calM \rightarrow \RR$ be two functions, and let $\rho$ be a probability measure supported on $\calM$. We define the inner product of $f$ and $g$ with respect to $\rho$ as $\langle f , g\rangle := \int_{\calM}  f(x) g(x)  d  \rho$. The $L^2$ norm of $f$ with respect to $\rho$ is $\|f\| := (\int_\calM  |f(x) |^2 d\rho)^{\frac 1 2}$. Given $n$ i.i.d. samples $\{x_i\}_{i=1}^n$ of $\rho$, the empirical $L^2$ norm of $f$ is $\|f\|_n := \frac 1 n \sum_{i=1}^n |f(x_i)|^2$.
The $L^\infty$ norm of $f$ is $ \|f\|_\infty := \sup\operatorname{ess} |f| $.
We denote probability and expectation by $\PP$ and $\EE$, respectively.
For a fixed $M>0$, $T_M$ is the truncation operator defined by $ T_M(x) := \min(|x|,M){\rm sign}(x).$
We denote by $\chijk$ the indicator function of an indexed set $\Cjk$ ({\it i.e.}, $\chijk(x) = 1$ if $x \in \Cjk$, and $0$ otherwise).

\paragraph{Setup}

We consider the problem of estimating a function
$f : \calM \rightarrow \RR$
given $n$ samples $\{(x_i, y_i)\}_{i=1}^n$,
where
\begin{itemize}
\item $\calM$ is an unknown Riemannian manifold of dimension $d$ isometrically embedded in $\RR^D$, with $d \ll D$;
\item $\rho$ is an unknown probability measure supported on $\calM$;
\item $\{x_i\}_{i=1}^n$ are independently drawn from $\rho $;
\item $ y_i = f(x_i) + \zeta_i $;
\item $\zeta_i$ are i.i.d. sub-Gaussian random variables, independent of $x_i$.
\end{itemize}
We wish to construct an estimator $\hf$ of $f$ minimizing the mean squared error
$${\rm MSE} := \EE \|f-\hf\|^2 =\EE \int_{\calM} |f(x)-\hf(x)|^2 d\rho .$$

\section{Geometric multi-resolution analysis}
\label{secgmra}

Geometric multi-resolution analysis (GMRA) is an efficient tool to build low-dimensional representations of data concentrating on or near a low-dimensional set embedded in high dimensions.
To keep the presentation self-contained,
we summarize here the main ideas,
and refer the reader to \cite{CM:MGM2,MMS:NoisyDictionaryLearning,LiaoMaggioni} for further details.
Given a probability measure $\rho$ supported on a $d$-dimensional manifold $\calM \subset \RR^D $, GMRA performs the following steps:
\begin{enumerate}
\item Construct a multiscale tree decomposition $\mathcal{T}$ of $\calM$ into nested cells $ \mathcal{T} := \{\Cjk\}_{k \in \calK_j, j \in \ZZ}$, where $j$ represents the scale and $k$ the location. Here $\calK_j$ is a location index set. 

\item Compute a local principal component analysis on each $\Cjk$. Let $\cjk$ be the mean of $x$ on $\Cjk$, and $\Vjk$ the $d$-dimensional principal subspace of $\Cjk$. Define $\calPjk : = \cjk +\proj_{\Vjk} (x-\cjk)$.

\end{enumerate}

An ideal multiscale tree decomposition should satisfy assumptions \ref{A1}$\div$\ref{A5} below for all integers $j \ge \jmin$:

\begin{enumerate}[label=\textnormal{(A\arabic*)}]
\item \label{A1} For every $k \in \calK_{j}$ and $k'\in \calK_{j+1}$, either $C_{j+1,k'}\subseteq C_{j,k}$ or $\rho(C_{j+1,k'}\cap C_{j,k}) = 0$.
The children of $C_{j,k}$ are the cells $ C_{j+1,k'} $ such that $ C_{j+1,k'} \subseteq C_{j,k} $.
We assume that $1\le\amin\le\#\{C_{j+1,k'}: C_{j+1,k'} \subseteq C_{j,k}\}\le\amax$ for all $k \in \calK_j$   and $j \ge \jmin$.
Also, for every $\Cjk$, there exists a unique $k' \in \calK_{j-1}$ such that $C_{j,k} \subseteq C_{j-1,k'}$. 
We call $C_{j-1,k'}$ the parent of $C_{j,k}$.

\item \label{A2} $\rho\left(\calM \setminus \bigcup_{k\in\calK_j}C_{j,k}\right) = 0$, {\it i.e.} $\Lam_j :=\{C_{j,k}\}_{k\in\calK_j}$ is a partition of $\calM$, up to negligible sets.

\item \label{A3} There exists $\theta_1>0$ such that $\#\Lam_j  \le 2^{jd}/\theta_1$.

\item \label{A4} There exists $\theta_2>0$ such that, if $x$ is drawn from $\rho$ conditioned on ${\Cjk}$, then $\|x-c_{j,k}\| \le \theta_2 2^{-j}$ almost surely.

\item \label{A5} Let $\lambda_1^{j,k} \ge \lambda_2^{j,k} \ge \ldots \ge\lambda_D^{j,k}$ be the eigenvalues of the covariance matrix $\Sjk$ of $\rho|_{\Cjk}$, defined in Table \ref{tab:gmra}. Then:
\begin{enumerate}[label=(\roman*)]
\item \label{A5i} there exists $\theta_3>0$ such that, for every $ j\ge \jmin$ and $k \in \calK_j$, $\lambda_d^{j,k} \ge \theta_3 {2^{-2j}}/{d}$;

\item \label{A5ii} there exists $ \theta_4 \in (0,1)$ such that $\lambda_{d+1}^{j,k} \le \theta_4 \lambda_{d}^{j,k}$. 
\end{enumerate}
\end{enumerate}

These are natural properties for multiscale partitions generalizing dyadic partitions to nonEuclidean domains \citep[see][]{MR1096400}.
\ref{A1} establishes that the cells constitute a tree structure.
\ref{A2} says that the cells at scale $j$ form a partition.
\ref{A3} guarantees that there are at most $2^{jd}/\theta_1$ cells at scale $j$.
\ref{A4} ensures that the diameter of all cells at scale $j$ is bounded by $ 2^{-j}$, up to a uniform constant.
\ref{A5}\ref{A5i} assumes that the best rank $d$ approximation to the covariance of a cell is close to the covariance matrix of a $d$-dimensional Euclidean ball,
while \ref{A5}\ref{A5ii} assumes that the cell has significantly larger variance in $d$ directions than in all the remaining ones.

Since all cells at scale $j$ have similar diameter,
$\Lam_j$ is called a uniform partition.
A master tree $\calT$ is a tree satisfying the properties above.
A proper subtree $\tcalT$ of $\calT$ is a collection of nodes of $\calT$ with the properties: the root node is in $\tcalT$;
if a node is in $\tcalT$, then its parent is also in $\tcalT$.
Any finite proper subtree $\tcalT$ is associated with a unique partition $\Lam = \Lam(\tcalT)$ consisting of its outer leaves,
by which we mean those nodes that are not in $\tcalT$, but whose parent is.

In practice, the master tree $\calT$ is not given.
We will construct one by an application of the cover tree algorithm \citep{LangfordICML06-CoverTree} (see \citep[Algorithm 3]{LiaoMaggioni}).
In order to make the samples for tree construction and function estimation independent from each other,
we split the data in half and use one subset to construct the tree
and the other one for local PCA and regression.
From now on we index the training data as $\{(x_i,y_i)\}_{i=1}^{2n}$,
and split them in $\{(x_i,y_i)\}_{i=1}^{2n} = \{(x_i,y_i)\}_{i=1}^{n}\cup \{(x_i,y_i)\}_{i=n+1}^{2n}$.
Running Algorithm \citep[Algorithm 3]{LiaoMaggioni} on $\{x_i\}_{i=n+1}^{2n}$,
we construct a family of cells $\{\hCjk\}_{k \in \calK_j, \jmin\le j \le \jmax}$
which satisfies \ref{A1}$\div$\ref{A4} with high probability if $\rho$ is doubling\footnote{$\rho$ is doubling if there exists $C_1>1$ such that  $C_1^{-1} r^d \le \rho(\calM \cap B_r(x)) \le C_1 r^d$ for any $x \in \calM$ and $r>0$; $C_1$ is called the doubling constant of $\rho$. See also \cite{MR1096400,DengHan}.};
furthermore, if $\calM$ is a $\calC^{s}$, $s \in (1,\infty)$, $d$-dimensional closed Riemannian manifold isometrically embedded in $\RR^D$,
and $\rho$ is the volume measure on $\calM$,
then \ref{A5} is satisfied as well:

\begin{proposition}[Proposition 14 in \cite{LiaoMaggioni}]
\label{propcovertree2}
Assume $\rho$ is a doubling probability measure on $\calM$ with doubling constant $C_1$. Then, the $\hCjk$'s constructed from \citep[Algorithm 3]{LiaoMaggioni} in satisfy:
\begin{enumerate}[label=\textnormal{(a\arabic*)}]
\item \ref{A1} with $\amax = C_1^2 (24)^d$ and $\amin=1$;

\item \label{a2}
let $\hcalM = \bigcup_{j = \jmin}^{\jmax} \bigcup_{k \in \calK_j}  \hCjk$;
for any $\nu>0$, 
$$
\PP\left\{\rho(\calM\setminus \hcalM) > \frac{28\nu\log n}{3n}\right\} \le 2n^{-\nu} ;
$$

\item \ref{A3} with $\theta_1 = C_1^{-1} 4^{-d}$;

\item \ref{A4} with $\theta_2 = 3$.
\end{enumerate}
If additionally $\calM$ is a $\calC^s, s \in (1,\infty)$, $d$-dimensional  closed  Riemannian manifold isometrically embedded in $\RR^D$,
and $\rho$ is the volume measure on $\calM$, then
\begin{enumerate}[label=\textnormal{(a5)}]
\item \ref{A5} is satisfied when $j$ is sufficiently large.
\end{enumerate}
\end{proposition}
Since there are finite training points, the constructed master tree has a finite number of nodes.  We first build a tree whose leaves contain a single point, and then prune it to the largest subtree whose leaves contain at least $d$ training points.  This pruned tree associated with the $\hCjk$'s is called the data master tree, and denoted by $\calT_n$. The $\hCjk$'s cover $\hcalM$, which represents the part of $\calM$ that has been explored by the data.
Even though assumption \ref{A2} is not exactly satisfied, we claim that \ref{a2} is sufficient for our performance guarantees, for example in the case that $\|f\|_\infty \le M $.
Indeed, simply estimating $f$ on $\calM \setminus \hcalM$ by $0$, for any $\nu>0$ we have
\begin{align*}
&\PP\left\{ \int_{\calM \setminus \hcalM} \|f \|^2 d\rho
\ge \frac{28M^2\nu \log n}{3n}
\right\}  \le 2n^{-\nu} \quad \text{and} \quad
\EE  \int_{\calM \setminus \hcalM} \|f \|^2 d\rho
\le \frac{56M^2\nu \log n}{3n^{1+\nu}}.
\end{align*}
In view of these bounds, the rate of convergence on $\calM\setminus \hcalM$ is faster than the ones we will obtain on $\hcalM$.
We will therefore assume \ref{A2}, thanks to \ref{a2}.
Also, it may happen that conditions \ref{A3}$\div$\ref{A5} are satisfied at the coarsest scales with very poor constants $\bf\theta$.
Nonetheless, it will be clear that in all that follows we may discard a few coarse scales, and only work at scales that are fine enough and for which \ref{A3}$\div$\ref{A5} truly capture in a quantitative way the local geometry of $\mathcal{M}$.
Since regression is performed on an independent subset of data,
we can assume, by conditioning, that the $\hCjk$'s are given and satisfy the required assumptions.
To keep the notation simple, from now on we will use $\Cjk$ instead of $\hCjk$, and $\calM$ in place of $\hcalM$, with a slight abuse of notation.

Besides cover tree, there are other methods that can be applied in practice to obtain multiscale partitions, such as METIS \citep{KarypisSIAM99-METIS}, used in \citet{CM:MGM2},
iterated PCA \citep{Szlam:iteratedpartitioning}),
and iterated $k$-means.
These methods can be computationally more efficient than cover tree,
but lead to partitions where the properties \ref{A1}$\div$\ref{A5} are not guaranteed to hold.

After constructing the multiscale tree $\calT$,
GMRA computes a collection of affine projectors $\{\calP_j: \RR^D \rightarrow \RR^D\}_{j \ge \jmin}$.
The main objects of GMRA in their population and sample version are summarized in Table \ref{tab:gmra}.
Given a suitable partition $\Lambda \subset \calT $,
$\calM$ can be approximated by the piecewise linear set $\{\calPjk(\Cjk)\}_{\Cjk\in\Lambda}$.

 \begin{center}
\begin{table}[h]
\renewcommand{\arraystretch}{2}
\resizebox{\columnwidth}{!}{
\begin{tabular}{ c  c  c  }
   & GMRA&  empirical GMRA 
   \\
   \cline{2-3}
  \multicolumn{1}{c}{measure} & $\rho(\Cjk)$  & $\hrho(\Cjk) := \frac{\hnjk}{n}, \quad \hnjk:= \#\{x_i: x_i\in\Cjk\}$\\ [5pt]
      \cline{2-3}
 \multicolumn{1}{c}{mean}  & $\cjk:= \frac{1}{\rho(\Cjk)} \int_{\Cjk} x d\rho $
  & 
  $ \hc_{j,k} := \frac{1}{\widehat{n}_{j,k}}\textstyle\sum_{x_i \in \Cjk} x_i $
  \\ [5pt]
       \cline{2-3}
       \multicolumn{1}{c}{covariance} & $\Sigma_{j,k} := \frac{1}{\rho(\Cjk)} \int_{\Cjk} (x-c_{j,k})(x-c_{j,k})^T d\rho $
& 
$\hS_{j,k} := \frac{1}{\hnjk} \sum_{x_i \in \Cjk} (x_i-\hc_{j,k})(x_i-\hc_{j,k})^T $
  \\
      [7pt]
       \cline{2-3}
    \multicolumn{1}{c}{\parbox{3cm}{\centering principal \\ subspace}}
    &\parbox{7cm}{\vspace{8pt} \centering $\Vjk $ minimizes \\  $ \frac{1}{\rho(\Cjk)} \int_{\Cjk} \|x-c_{j,k}-\proj_V(x-c_{j,k})\|^2 d\rho $ \\ over $d$-dim subspaces $V$
    }
        & \parbox{7cm}{\vspace{8pt}\centering
  $\hV_{j,k} $ minimizes
  \\ { $\frac{1}{\hn_{j,k}} 
 \textstyle \sum_{x_i \in \Cjk} \|x-\hc_{j,k} -\proj_V (x-\hc_{j,k})\|^2$}
  \\
 over $d$-dim subspaces $V$}
      \\ [22pt]
       \cline{2-3}
\multicolumn{1}{c}{\parbox{3cm}{\centering projection}}
  & 
  $\calP_{j,k}(x) := c_{j,k} +\proj_{V_{j,k}} (x-c_{j,k})$
  &   $\hcalP_{j,k}(x) := \hcjk +\proj_{\hVjk} (x-\hcjk)$
  \\
       [7pt]
       \cline{2-3}
\end{tabular}
}
\caption{
Objects of GMRA and their empirical counterparts.
$V_{j,k}$ and $\hVjk$ are the eigenspaces associated with the largest $d$ eigenvalues of $\Sjk$ and $\hSjk$, respectively.
}
\label{tab:gmra}
\end{table}
\end{center}

\section{Multiscale polynomial regression}
\label{secreg}

Given a multiscale tree decomposition $\{\Cjk\}_{j,k}$
and training samples $\{ (x_i,y_i)\}_{i=1}^n$,
we construct a family $ \{\hfjk\}_{j,k}$ of local estimates of $f$ in two stages:
first we compute local coordinates on $\Cjk$ using GMRA outlined above,
and then we estimate $f_{|\Cjk}$ by fitting a polynomial of order $\ell$ on such coordinates.
A global estimator $\hf^\ell_\Lam$ is finally obtained by summing the local estimates over a suitable partition $\Lam$.
Our regression method is detailed in Algorithm \ref{AlgorithmRegression}.
In this section we assume that $f$ is bounded, with $\|f\|_\infty \le M$.

\begin{algorithm}[t!]                      	
\caption{GMRA regression}          	
\label{AlgorithmRegression}		
\begin{algorithmic}[1]                    	
    \REQUIRE  training data $\{x_i,y_i\}_{i=1}^{2n}$, intrinsic dimension $d$, bound $M$, polynomial order $\ell$, approximation type (uniform or adaptive).
\vspace{5pt}
    \ENSURE multiscale tree decomposition $ \calT_n $,
    partition $\Lam$,
    piecewise $\ell$-order polynomial estimator $\hf_\Lam^\ell$.
\vspace{5pt}
    \STATE construct a multiscale tree $\calT_n$ by \citep[Algorithm 3]{LiaoMaggioni} on $\{x_i\}_{i=n+1}^{2n}$; \vspace{-12pt}
    \STATE compute centers $\hcjk$ and subspaces $\hVjk$ by empirical GMRA on $\{x_i\}_{i=1}^n$; \vspace{3pt}
    \STATE define coordinates $\hpijk$ on $\Cjk$:
    \vspace{-5pt}
$$\hpijk : \Cjk \to \RR^d , \qquad \hpijk(x) := \hVjk^T (x-\hcjk) ;
$$
    \vspace{-15pt}
    \STATE compute local estimators $\widehat{g}_{j,k}^\ell$ by solving the following least squares problems over the space $P^\ell$ of polynomials of degree $\le \ell$:
        \vspace{-5pt}
\begin{align}
 &\widehat{p}_{j,k}^\ell := \argmin\limits_{p \in P^\ell} \frac{1}{\hnjk} \sum_{i=1}^n | y_i - p\circ \hpijk(x_i) |^2 \chijk(x_i) , \qquad \widehat{g}_{j,k}^\ell := \widehat{p}_{j,k}^\ell \circ \hpijk ;
 \label{eqls1} 
\end{align}
    \vspace{-5pt}
         \STATE truncate $\hgjk^\ell$ by $M$:
         $$ \hfjkell := T_M[\hgjkell] ; $$
             \vspace{-10pt}
    \STATE construct a uniform (see Section \ref{sec:uniform}) or adaptive (see Section \ref{sec:adapt}) partition $\Lam$;
    \STATE define the global estimator $\hf_\Lam^\ell$ by summing the local estimators over the partition $\Lam$:
    $$ \hf^\ell_\Lam := \sum_{\Cjk\in\Lam} \hfjk^\ell \chijk. $$
   \end{algorithmic}
\end{algorithm}

In order to analyze the performance of our method,
we introduce the oracle estimator $f^\ell_\Lam$
based on the distribution $\rho$, defined by
\begin{align*}
& \pijk : \Cjk \to \RR^d, \qquad \pijk(x) := V_{j,k}^T (x-\cjk), 
\\
&p_{j,k}^\ell := \argmin\limits_{p \in P^\ell} \int_{\Cjk} | y - p\circ \pijk(x) |^2 d\rho,
 \\
 &f_{j,k}^\ell := T_M[p_{j,k}^\ell \circ \pijk]
 \\
 &f^\ell_\Lam := \sum_{\Cjk\in\Lam} f_{j,k}^\ell \chijk ,
\end{align*}
and split the MSE into a bias and a variance term:
\begin{equation} \label{eq:bias-variance}
\EE\|f - \hf_\Lam^\ell\|^2 \le  2\underbrace{\|f - f_\Lam^\ell\|^2}_{\text{bias}^2} + 2 \underbrace{\EE\|f_\Lam^\ell - \hf_\Lam^\ell\|^2}_{\text{variance}} .
\end{equation}
The bias term is a deterministic approximation error, and will be handled by assuming suitable regularity models for $\rho$ and $f$
(see Definitions \ref{def:As} and \ref{def:Bs}).
The variance term quantifies the stochastic error arising from finite-sample estimation,
and will be bounded using concentration inequalities (see Proposition \ref{prop:var-bound}).
The role of $\Lam$, encoded in its size $\#\Lam$, is crucial to balance \eqref{eq:bias-variance}.
We will discuss two possible choices:
uniform partitions in Section \ref{sec:uniform},
and adaptive at multiple scales in Section \ref{sec:adapt}.

\subsection{Uniform partitions}
\label{sec:uniform}
 
A first natural choice for $\Lam$ is a uniform partition $\Lam_j := \{\Cjk\}_{k \in \calK_j},  {j \ge \jmin}$.
At scale $j$, $f$ is estimated by $\hf_{\Lam_j}^\ell = \sum_{k \in \calK_j} \hfjk^\ell \chijk$.
The bias $\|f - f_{\Lam_j}^\ell\|$ decays at a rate depending on the regularity of $f$, which can be quantified as follows:

\begin{definition}[model class $\ASL$] 
\label{def:As}
A function $f: \calM \rightarrow \RR$ is in the class $\ASL$ for some $ s > 0 $ with respect to the measure $\rho$ if
$$ |f|_{\ASL} := \sup_{\calT} \ \sup_{j\ge\jmin} \ \frac{\|f-\fj^{\ell}\|}{2^{-js}} < \infty , $$
where $\calT$ ranges over the set, assumed non-empty, of multiscale tree decompositions satisfying assumptions \ref{A1}$\div$\ref{A5}.
\end{definition}
We capture the case where the bias is roughly the same on every cell with the following definition:
\begin{definition}[model class $\ASLINF$]
\label{modelaslinf}
A function $f: \calM \rightarrow \RR$ is in the class $\ASLINF$ for some $ s > 0 $ with respect to the measure $\rho$ if
$$ |f|_{\ASLINF} := \sup_{\calT} \ \sup_{j\ge\jmin} \ \sup_{k \in \calK_j} \ \frac{\|(f -\fjkell)\chijk \|}{2^{-js}\sqrt{\rho(\Cjk)}} < \infty , $$
where $\calT$ ranges over the set, assumed non-empty, of multiscale tree decompositions satisfying assumptions \ref{A1}$\div$\ref{A5}.
\end{definition}
Clearly $\ASLINF \subset \ASL$. These classes contain uniformly regular functions on manifolds, such as H\"older functions.

\begin{example} \label{hold-As}
Let $\calM$ be a closed smooth $d$-dimensional Riemannian manifold isometrically embedded in $\RR^D$, and let $\rho$ be the volume measure on $\calM$. Consider a function $f :\calM \rightarrow \RR$ and a smooth chart $(U,\phi)$ on $\calM$. The function $\tilde f : \phi(U) \rightarrow \RR$ defined by $\tilde f(v) = f\circ \phi^{-1}(v)$ is called the coordinate representation of $f$.
Let $\lambda = (\lambda_1,\ldots,\lambda_d)$ be a multi-index with $|\lambda| := \lambda_1+\ldots+\lambda_d = \ell$. The $\ell$-order $\lambda$-derivative of $f$ is defined as 
$$\partial^\lambda f(x) := \partial^\lambda (f \circ \phi^{-1}). $$
H\"older functions $\calC^{\ell,\alpha}$ on $\calM$ with $ \ell \in \NN $ and $ \alpha \in (0,1] $ are defined as follows: $f \in \calC^{\ell,\alpha}$ if the $\ell$-order derivatives of $f$ exist, and 
$$ |f|_{\calC^{\ell,\alpha}} := \max_{|\lambda| = \ell} \sup_{x \ne z} \frac{|\partial^\lambda f(x) - \partial^\lambda f(z)|}{d(x,z)^\alpha} < \infty , $$
$d(x,z)$ being the geodesic distance between $x$ and $z$.
We will always assume to work at sufficiently fine scales at which $ d(x,z) \asymp  \| z - x \|_{\RR^D} $.
Note that $\calC^{\ell,1}$ is the space of $\ell$-times continuously differentiable functions on $\calM$ with Lipschitz $\ell$-order derivatives.
We have $ \calC^{\ell,\alpha} \subset \calA_{\ell+\alpha}^{\ell,\infty} $
with $ |f|_{\calA_{\ell+\alpha}^{\ell,\infty}} \le \theta_2^{\ell+\alpha} {d^\ell} |f|_{\calC^{\ell,\alpha}} / {\ell!} $. 
The proof is in Appendix \ref{sec:proofs_lemmata}.
\end{example}

\begin{example}
\label{example1}
Let $\calM$ be a smooth closed Riemannian manifold isometrically embedded in $\RR^D$, and let $\rho$ be the volume measure on $\calM$. 
Let $\Omega\subset \calM$ such that $\Gamma:=\partial\Omega$ is a smooth and closed $d_\Gamma$-dimensional submanifold with finite reach\footnote{The reach of $\calM$ is an important global characteristic of $\calM$. Let 
$D(\calM) := \{ y \in \RR^D: \exists !  \ x\in \calM \text{ s.t. } \|x-y\| = \inf_{z\in \calM} \|z-y\|\}$,
$\calM_r := \{ y\in \RR^D: \inf_{x \in \calM} \|x-y\| < r\}$.
Then  ${\rm reach}(\calM) := \sup \{r\ge 0: \calM_r \subset D(\calM)\}$.
See also \cite{federer1959curvature}.
}. 
Let $g = a \mathbf{1}_{\Omega} + b \mathbf{1}_{\Omega^\complement}$ for some $a,b\in\RR$, where $ \mathbf{1}_{S} $ denotes the indicator function of  a set $S$. Then $g\in \calA^{\ell}_{(d-d_\Gamma)/{2}}$ for every $\ell =0,1,2,\ldots$; however, $g \notin \calA^{\ell,\infty}_s$ for any $s>0$. The proof is in Appendix \ref{sec:proofs_lemmata}.
\end{example}

When we take uniform partitions $\Lam = \Lam_j $ in \eqref{eq:bias-variance}, the squared bias satisfies $$\|f - f_{\Lam_j}^\ell\|^2 \le |f|_{\ASL}^2 2^{-2js}$$ whenever $f \in \ASL$, which decreases as $j$ increases.
On the other hand, Proposition \ref{prop:var-bound} shows that the variance at the scale $j$ satisfies
$$\EE\|f_{\Lam_j}^\ell - \hf_{\Lam_j}^\ell\|^2 \le   O \left(\frac{j2^{jd}}{n} \right) ,
$$
which increases as $j$ increases.
Choosing the optimal scale $j^\star$ in the bias-variance tradeoff,
we obtain the following rate of convergence for uniform estimators:
\begin{theorem} \label{thm:unif-bound}
Suppose $\|f\|_\infty \le M$ and $ f \in   \ASL $ for $\ell \in \{0,1\}$ and $s>0$. Let $j^\star$ be chosen such that $$ 2^{-j^\star} := \mu \left(\lognn\right)^{\frac{1}{2s+d}} $$
for $\mu>0$.
Then there exist positive constants $c:=c(\theta_1,d,\mu)$ and $C:= C( \theta_1,d,\mu) $ for $\ell = 0$, or 
$c:=c(\theta_1,\theta_2,\theta_3,d,\mu)$ and $C:= C( \theta_1,\theta_2,\theta_3,d,\mu) $ for $\ell = 1$, such that:
\begin{enumerate}[label=\textnormal{(\alph*)},leftmargin=*]
\item \label{unif-bound-prob} for every $ \nu > 0 $ there is $c_\nu>0$ such that
$$ \PP\left\{\|f-\hf^\ell_{\Lam_{j^\star}}\| >(|f|_{\ASL} \mu^s + c_\nu) \left(\lognn \right)^{\frac{s}{2s+d}}\right\} \le C n^{-\nu} , $$
where $ c_\nu:=c_\nu(\nu,\theta_1,d,M,\sigma,s,\mu)  $ for $\ell =0$, and $ c_\nu:=c_\nu(\nu,\theta_1,\theta_2,\theta_3,d,M,\sigma,s,\mu)  $ for $\ell =1$;
\item \label{unif-bound-exp} 
$ \EE\|f-\hf^\ell_{\Lam_{j^\star}}\|^2 \le \left(|f|^2_{\ASL}\mu^s + c\max(M^2,\sigma^2)\right) \left( \lognn\right)^{\frac{2s}{2s+d}} $.
\end{enumerate}
\end{theorem}

Theorem \ref{thm:unif-bound} is proved in Section \ref{sec:proofs}.
Note that the rate depends on the intrinsic dimension $d$ instead of the ambient dimension $D$.
Moreover, the rate is optimal (up to logarithmic factors) at least in the case of $\calC^{\ell,\alpha}$ functions on $\mathcal{M}$, as discussed in Example \ref{hold-As}.

\subsection{Adaptive partitions} \label{sec:adapt}

Theorem \ref{thm:unif-bound} is not fully satisfactory for two reasons: (i) the choice of optimal scale requires knowledge of the regularity of the unknown function; (ii) no uniform scale can be optimal if the regularity of the function varies at different locations and scales.
Inspired by \citet{BCDDT1,BCDD2}, we thus propose an adaptive estimator which learns near-optimal partitions from data, without knowing the possibly nonuniform regularity of the function.

\begin{table}[h]
\renewcommand{\arraystretch}{1.8}
\centering
\begin{tabular}{ c c }
   oracles &  empirical counterparts \\
   [5pt]
\hline
  $\Wjk^\ell := (\fj^\ell-f^\ell_{\Lam_{j+1}})\chijk$
  & $\hWjk^\ell := (\hfj^\ell-\hf^\ell_{\Lam_{j+1}})\chijk$
  \\
     [5pt]
  \hline
  $\Deltajkell := \|\Wjk^\ell\| $
  & $\hDeltajkell := \|\hWjk^\ell\|_n $
   \\ [5pt]   \hline
    \end{tabular}
\caption{Approximation difference in refining $\Cjk$}
\label{tab:refinement}
\end{table}

Adaptive partitions may be selected by a criterion that determines whether or not a cell should be picked or not.
The quantities involved in this selection are summarized in Table \ref{tab:refinement}, along with their empirical versions.
\smash{$\Deltajkell$} measures the local difference in approximation between two consecutive scales:
a large \smash{$\Deltajkell$} suggests a significant reduction of error if we refine $\Cjk$ to its children.
Intuitively, we should truncate the master tree to the subtree including the nodes where this quantity is large.
However, if too few samples exist in a node, then the empirical counterpart \smash{$\hDeltajkell$} can not be trusted.
We thus proceed as follows.
We set a threshold $\tau_n$ decreasing in $n$,
and let $\hcalT_n(\tau_n)$ be the smallest proper subtree of $\calT_n$ containing all $\Cjk$'s for which $\smash{\hDeltajkell} \ge \tau_n$.
Crucially, $\tau_n$ may be chosen independently of the regularity of $f$ (see Theorem \ref{thm:adapt-bound}).
We finally define our adaptive partition $\hLam_n(\tau_n)$ as the partition associated with the outer leaves of $\hcalT_n(\tau_n)$.
The procedure is summarized in Algorithm \ref{alg-adapt}.

 \begin{algorithm}[h]                      	
\caption{Adaptive partition}          	
\label{alg-adapt}		
\begin{algorithmic}[1]        
    \REQUIRE  training data $\{(x_i,y_i)\}_{i=1}^{n}$, multiscale tree decomposition $\calT_n$,  local $\ell$-order polynomial estimates $\{\hfjk^\ell\}_{j,k}$, threshold parameter $\kappa$.
    \ENSURE adaptive partition $\hLam_n(\tau_n)$.
    \STATE compute the approximation difference $\hDeltajkell$ on every node $\Cjk \in \calT_n$;
    \STATE set the threshold $\tau_n := \kappa\sqrt{(\log n)/n}$;
    \STATE select the smallest proper subtree $\hcalT_n(\tau_n) $ of $\calT_n$ containing all $\Cjk$'s with $\hDeltajkell \ge \tau_n$;
     \STATE define the adaptive partition $\hLam_n(\tau_n)$ associated with the outer leaves of $\hcalT_n(\tau_n)$.
\end{algorithmic}
\end{algorithm}

To provide performance guarantees for our adaptive estimator,
we need to define a proper model class based on oracles.
Given any master tree $\calT$ satisfying assumptions \ref{A1}$\div$\ref{A5}
and a threshold $ \tau > 0 $,
we let $ \calT(\tau) $ be the smallest subtree of $\calT$ consisting of all the cells $\Cjk$'s with $ \Deltajkell \ge \tau $.
The partition made of the outer leaves of $\calT(\tau)$ is denoted by $\Lam(\tau)$.

\begin{definition}[model class $\BSL$]
 \label{def:Bs}
A function $f: \calM \rightarrow \RR$ is in the class $\BSL$ for some $ s > 0 $ with respect to the measure $\rho$ if
$$ |f|_{\BSL}^p := \sup_\calT \sup_{\tau>0} \tau^p \#\calT(\tau) <\infty, \qquad p = \frac{2d}{2s+d} , $$
where $\calT$ varies over the set, assumed non-empty, of multiscale tree decompositions satisfying assumptions \ref{A1}$\div$\ref{A5}.
\end{definition}
In general, the truncated tree $\calT(\tau)$ grows as the threshold $\tau$ decreases. For elements in $\BSL$, we have control on the growth rate, namely $\#\calT(\tau_n) \lesssim \tau^{-p}$. The class $\BSL$ is indeed rich, and contains in particular $\ASLINF$,
while additionally capturing functions of nonuniform regularity.

\begin{lemma} \label{AsInfBs}
$\ASLINF \subset \BSL$. If $f \in \ASLINF$, then $f \in \BSL$ and $|f|_{\BSL} \le (\amin/\theta_1)^{\frac{2s+d}{2d}} |f|_{\ASLINF}$.
\end{lemma}
The proof is given in Appendix \ref{sec:proofs_lemmata}.

\begin{example}
\label{example1a} Let $g$ be the function in Example \ref{example1}. Then $g \in \calB^{\ell}_{d(d-d_\Gamma)/(2d_\Gamma)}$ for every $\ell = 0,1,2,\ldots$. Notice that $g\in \calA^{\ell}_{(d-d_\Gamma)/{2}}$, so $g$ has a larger regularity parameter $s$ in the $\calB^\ell_s$ model than in the $\calA^\ell_s$ model.  
\end{example}

We will also need a quasi-orthogonality condition ensuring that the functions $\Wjk^\ell$ representing the approximation difference between two scales are almost orthogonal across scales.
\begin{definition} \label{def:quasi-ortho}
We say that $f$ satisfies quasi-orthogonality of order $\ell$ with respect to the measure $\rho$ if there exists a constant $B_0 > 0$ such that, for any proper subtree $\calS$ of any tree $\calT$ satisfying assumptions \ref{A1}$\div$\ref{A5},
 $$ \Bigl\| \sum_{\Cjk\in \calT \setminus \calS} \Wjk^\ell \Bigr\|^2 \le B_0 \sum_{\Cjk \in \calT\setminus \calS} \|\Wjk^\ell\|^2 . $$
\end{definition}

The following lemma shows that $f \in \BSL$, along with quasi-orthogonality, implies a certain approximation rate of $f$ by $f^\ell_{\Lam(\tau)}$ as $\tau \rightarrow 0$.
The proof is given in Appendix \ref{sec:proofs_lemmata}.

\begin{lemma} \label{lem:bias-Bs}
 If $ f \in \BSL \cap (L^\infty \cup \calA_t^\ell) $ for some $s,t>0$, and $f$ satisfies quasi-orthogonality of order $\ell$, then
 $$ \|f-f^\ell_{\Lam(\tau)}\|^2 \le B_{s,d}|f|_{\calB_s}^p \tau^{2-p}\le B_{s,d}|f|_{\calB_s}^2 \#\Lam(\tau)^{-\frac{2s}{d}} , \quad p = \frac{2d}{2s+d} , $$
with $B_{s,d} := B_0 2^p \sum_{\ell\ge 0} 2^{-\ell(2-p)}$.
\end{lemma}

The main result of this paper is the following performance analysis of adaptive estimators, which is proved in Section \ref{sec:proofs}.

\begin{theorem} \label{thm:adapt-bound}
Let $\ell \in \{0,1\}$ and $M>0$.
Suppose $ \|f\| \le M$ and $f$ satisfies quasi-orthogonality of order $\ell$.
Set $ \tau_n := \kappa\sqrt{\log n/n} $. Then:
\begin{enumerate}[label=\textnormal{(\alph*)},leftmargin=*]
 \item
  For every $ \nu > 0 $ there exists $ \kappa_\nu := \kappa_\nu(\amax,\theta_2,\theta_3,d,M,\sigma,\nu) > 0 $ such that,
  whenever $ f \in \BSL $ for some $s >0$ and $ \kappa \ge \kappa_\nu $, there are $ r , C > 0 $ such that
  $$
   \PP\left\{ \|f-\hf^\ell_{\hLam_n(\tau_n)}\| > r \left(\lognn \right)^{\frac{s}{2s+d}} \right\} \le C n^{-\nu}\,.
   $$
 \item
  There exists $\kappa_0 : = \kappa_0(\amax,\theta_2,\theta_3,d,M,\sigma)$ such that, whenever $ f \in \BSL $ for some $s>0$ and $ \kappa \ge \kappa_0 $,
 there is $ \bar{C} > 0 $ such that  
  $$ \EE\|f -\hf^\ell_{\hLam_n(\tau_n)}\|^2 \le \bar{C} \left(\lognn \right)^{\frac{2s}{2s+d}} .$$
\end{enumerate}
Here $r$ depends on $ \theta_2$, $\theta_3$, $\amax$, $d$, $M$, $s$, $ {|f|}_{\BSL}$, $\sigma$, $B_0$, $\nu$, $\kappa $;
$C$ depends on $\theta_2$, $\theta_3$, $\amax$, $\amin$, $d$, $s$, $|f|_{\BSL}$, $\kappa$;
$\bar{C}$ depends on $ \theta_2$, $\theta_3$, $\amax$, $\amin$, $d$, $M$, $s$, $|f|_{\BSL}$, $B_0$, $\kappa $.
\end{theorem}
Theorem \ref{thm:adapt-bound} is more satisfactory than Theorem \ref{thm:unif-bound} for two reasons: (i) the same rate is achieved for a richer model class; (ii) the estimator does not require a priori knowledge of the regularity of the function, since the choice of $\kappa$ is independent of $s$.

For a given accuracy $\varepsilon$, in order to achieve ${\rm MSE} \lesssim \varepsilon^2$, the number of samples we need is $n_{\varepsilon} \gtrsim (1/\varepsilon)^{\frac{2s+d}{s}}\log(1/\varepsilon)$. When $s$ is unknown, we can determine $s$ as follows: we fix a small $n_0$, and run Algorithm \ref{alg-adapt} with $2n_0, 4n_0, \ldots, 2^jn_0,\dots$ samples. For each sample size, we evenly split data into a training set to build the adaptive estimator, and a test set to evaluate the MSE. According to Theorem \ref{thm:adapt-bound}, the MSE scales as $(\log n/n)^{\frac{2s}{2s+d}}$. Therefore, the slope in the log-log plot of the MSE versus $n$ gives an approximation of $-2s/(2s+d)$. This could be formalized by a suitable adaptation of Lepski's method.

\section{Computational considerations} \label{sec:computational}
\label{seccomputation}

The computational cost of Algorithms \ref{AlgorithmRegression} and \ref{alg-adapt} may be split as follows:
\ \\
\ \\
\noindent \textit {Tree construction.} Cover tree itself is an online algorithm where a single-point insertion or removal takes cost at most $O(\log n)$. The total computational cost of the cover tree algorithm is $C^d D n \log n$, where $C>0$ is a constant \citep{LangfordICML06-CoverTree}. 
\ \\

\noindent \textit {Local PCA.} At every scale $j$, we perform local PCA on the training data restricted to the $\Cjk$ for every $k \in \calK_j$ using the random PCA algorithm \citep{halko2009finding}. Recall that $\hnjk$ denotes the number of training points in $\Cjk$. The cost of local PCA at scale $j$ is in the order of $\sum_{k \in \calK_j} D d \hnjk = Dd n$, and there are at most $c\log n$ scales where $c>0$ is a constant, which gives a total cost of $cD d n \log n$. 
\ \\

\noindent \textit {Multiscale regression.} Given $\hnjk$ training points on $\Cjk$, computing the low-dimensional coordinates $\hpijk(x_i)$ for all $x_i \in \Cjk$ costs $Dd \hnjk$, and solving the linear least squares problem \eqref{eqls1}, where the matrix is of size $\hnjk \times d^\ell$, costs at most $\hnjk d^{2\ell}$. Hence, constructing the $\ell$-order polynomials at scale $j$ takes $\sum_{k \in \calK_j} Dd\hnjk+d^{2\ell}\hnjk = (Dd +d^{2\ell})n$, and there are at most $c\log n$ scales, which sums up to $c(Dd +d^{2\ell})n\log n$. 
\ \\

\noindent \textit {Adaptive approximation:}
We need to compute the coefficients $\hDeltajk$ for every $\Cjk$, which costs $2(Dd+d^\ell)\hnjk$ on $\Cjk$, and $2c(Dd+d^\ell)n\log n$ for the whole tree.
\ \\

\noindent In summary, the total cost of constructing GMRA adaptive estimators of order $\ell$ is
$$C^d D n\log n + (4Dd + d^{2\ell}+d^\ell)c n \log n , $$
which scales linearly with the number of samples $n$ up to a logarithmic factor.

\section{Proofs} \label{sec:proofs}

We analyze the error of our estimator by a bias-variance decomposition as in \eqref{eq:bias-variance}.
We present the variance estimate in Section \ref{secproof:variance},
the proofs for uniform approximations in Section \ref{secproof:uniform},
and for adaptive approximations in Section \ref{secproof:adapt}.

\subsection{Variance estimate}
\label{secproof:variance}

\begin{table}[h]
\renewcommand{\arraystretch}{2}
\begin{center}
\resizebox{\columnwidth}{!}{
\begin{tabular}{ c | c | }
  \multicolumn{2}{c}{piecewise constant: $\ell =0$}
  \\ \hline
  \multicolumn{1}{|c|}{oracles estimators} & {empirical counterparts }
  \\ \hline
  \multicolumn{1}{|c|}{$g_{j,k}^0(x) := \yjk:= \frac{1}{\rho(\Cjk)} \displaystyle\int_{\Cjk} y d\rho $} & $ {\widehat g}_{j,k}^0(x) := \hyjk := \frac{1}{\hnjk} \displaystyle\sum_{x_i \in \Cjk} y_i $
  \\ \hline
  \multicolumn{1}{|c|}{$ f_{j,k}^0(x) := T_M[g_{j,k}^0(x)] $} & $ \hf_{j,k}^0(x) := T_M[\hg_{j,k}^0(x)] $
  \\ \hline
  \multicolumn{1}{|c|}{$ f_{j,k}^0(x) := g_{j,k}^0(x) $} & $ \hf_{j,k}^0(x) := \hg_{j,k}^0(x) $
  \\ \hline
  \multicolumn{2}{c}{piecewise linear: $\ell =1$}
  \\ \hline
  \multicolumn{1}{|c|}{$ g_{j,k}^1(x) := [\pijk(x)^T \ 2^{-j}]\betajk $} & $\hg_{j,k}^1(x) := [\hpijk(x)^T \ 2^{-j}]\hbetajk $
   \\ [5pt] \cdashline{1-2}
  \multicolumn{1}{|c|}{$ \betajk := \left[\begin{smallmatrix} [\Lamjk]_d^{-1} & 0 \\ 0 & 2^{2j} \end{smallmatrix}\right] \frac{1}{\rho(\Cjk)} \displaystyle\int_{\Cjk} y \left[\begin{smallmatrix} \pijk(x) \\ 2^{-j} \end{smallmatrix}\right] d\rho $} & $ \hbetajk := \left[\begin{smallmatrix} [\hLamjk]_d^{-1} & 0 \\ 0 & 2^{2j} \end{smallmatrix}\right] \frac{1}{\hnjk} \displaystyle\sum_{x_i\in\Cjk} y_i \left[\begin{smallmatrix} \hpijk(x_i) \\ 2^{-j} \end{smallmatrix}\right] $
    \\ [5pt]
   \multicolumn{1}{|c|}{$ [\Lamjk]_d := \operatorname{diag}(\lamjk_1,\dots,\lamjk_d) $} & $ [\hLamjk]_d := \operatorname{diag}(\hlamjk_1,\dots,\hlamjk_d) $
    \\ [5pt] \hline
   \multicolumn{1}{|c|}{$ f_{j,k}^1(x) := T_M[ g_{j,k}^1(x)] $} & $ \hf_{j,k}^1(x) := T_M[\hg_{j,k}^1(x)] $
    \\ \hline
    \multicolumn{1}{|c|}{$ f_{j,k}^1(x) := \yjk + T_{L(\theta_2 2^{-j})^{\beta}}[g_{j,k}^1(x)-\yjk] $} & $ \hf_{j,k}^1(x) := \hyjk + T_{L(\theta_2 2^{-j})^{\beta}}[\hg_{j,k}^1(x)-\hyjk] $
    \\ \hline
\end{tabular}%
}
\caption{
Local constant and linear estimators on $\Cjk$.
The truncation in $[\Lamjk]_d$ has the effect of regularizing the least squares problem, which is ill-posed  due to the small eigenvalues $\{\lamjk_l\}_{l=d+1}^D$.
}
\label{tab:regression}
\end{center}
\end{table}

The main quantities involved in the $0$-order (piecewise constant) and the $1$st-order (piecewise linear) estimators are summarized in Table \ref{tab:regression}.
\begin{proposition} \label{prop:var-bound}
Suppose $\|f\|_\infty \le M$ and let $ \ell \in \{0,1\} $.
For any partition $\Lam$, let $f_\Lam^\ell$ and $\hf_\Lam^\ell$ be the optimal approximation and the empirical estimators of order $\ell$ on $\Lam$, respectively.
Then, for every $ \eta > 0 $,
\begin{equation}
\PP \left\{ \| f_\Lam^\ell - \widehat{f}_\Lam^\ell \| > \eta \right\} 
   \le \begin{cases}
         C_0 \#\Lam \exp\left( - \frac{n\eta^2}{c_0 \max(M^2,\sigma^2) \#\Lam} \right) & \text{ for $ \ell = 0 $} \\
         C_1 d \#\Lam \exp\left(- \frac{n\eta^2}{c_1 \max( d^4 M^2 , d^2 \sigma^2)\#\Lam}\right) & \text{ for $ \ell = 1 $}
        \end{cases}
\label{var-bound-prob}        
\end{equation}
and therefore
\begin{equation}
\EE \| f_\Lam^\ell - \widehat{f}_\Lam^\ell \|^2
    \le \begin{cases}
          \frac{c_0 \max( M^2 , \sigma^2) \#\Lam \log (C_0 \#\Lam )}{n} & \text{ for $ \ell = 0 $} \\
          \frac{c_1 \max( d^4 M^2 , d^2 \sigma^2) \#\Lam \log (C_1 d \#\Lam )}{n} & \text{ for $ \ell = 1 $}
         \end{cases}
\label{var-bound-exp}
\end{equation}
for some absolute constants $ c_0 , C_0 $ and some $ c_1 , C_1 $ depending on $\theta_2,\theta_3$.
\end{proposition}
\begin{proof}
Since $f_\Lam^\ell$ and $\hf_\Lam^\ell$ are bounded by $M$,
we define $ \Lam^{-} := \{ \Cjk\in\Lam:\rho(\Cjk) \le \frac{\eta^2}{4M^2\#\Lam} \} $, and observe that
$$ \sum_{\Cjk\in \Lam^{-}} \| (f_\Lam^\ell - \hf_\Lam^\ell) \chijk \|^2 \le \eta^2 . $$
We then restrict our attention to $\Lam^{+} := \Lam\setminus \Lam^{-} $  and apply Lemma \ref{lem:fjk} with $ t = \frac{\eta}{\sqrt{\rho(\Cjk)\#\Lam}}$.
This leads to \ref{var-bound-prob}, and
\ref{var-bound-exp} follows from \ref{var-bound-prob} by integrating over $\eta>0$.
\end{proof}

\subsection{Proof of Theorem \ref{thm:unif-bound}}
\label{secproof:uniform}

Notice that $\#\Lam_j \le 2^{jd}/\theta_1$ by \ref{A3}. By choosing $j^\star$ such that $2^{-j^\star} = \mu \left(\lognn\right)^{\frac{1}{2s+d}}$ for some $\mu >0$, we have
$$\|f - f_{\Lam_{j^\star}}^\ell\| \le |f|_{\ASL} 2^{-j^\star s} 
\le  |f|_{\ASL} \mu^s \left( \lognn\right)^{\frac{s}{2s+d}}.$$
The probability estimate in Theorem \ref{thm:unif-bound} (a) follows from 
\begin{align*}
\PP\left\{ \|f^\ell_{\Lam_{j^\star}} -\hf^\ell_{\Lam_{j^\star}} \|\ge c_\nu \left(\lognn\right)^{\frac{s}{2s+d}}\right\}
& 
\le \begin{cases}
\frac{C_0}{\theta_1 \mu^d} (\log n)^{-\frac{d}{2s+d}} n^{- \left(\frac{\theta_1 \mu^d c_\nu^2}{c_0\max(M^2,\sigma^2)} -\frac{d}{2s+d} \right)}
& \ell = 0
\\
\frac{C_1 d}{\theta_1 \mu^d} (\log n)^{-\frac{d}{2s+d}} n^{- \left(\frac{\theta_1 \mu^d c_\nu^2}{c_1\max(d^4 M^2,d^2\sigma^2)} -\frac{d}{2s+d} \right)}
& \ell = 1
\end{cases}
\\
&\le C n^{-\nu}
\end{align*}
provided that $\frac{\theta_1 \mu^d c_\nu^2}{c_0\max(M^2,\sigma^2)} -\frac{d}{2s+d}  > \nu$ for $\ell = 0$ and $\frac{\theta_1 \mu^d c_\nu^2}{c_1\max(d^4 M^2,d^2\sigma^2)} -\frac{d}{2s+d}>\nu$ for $\ell =1$.

\hfill $\blacksquare$

\subsection{Proof of Theorem \ref{thm:adapt-bound}}
\label{secproof:adapt}

We begin by defining several objects of interest:
\begin{itemize}
\item $\calT_n$: the data master tree whose leaves contain at least $d$ points of training data. It can be viewed as the part of a multiscale tree that our training data have explored. Notice that 
$$
\#\calT_n \le \sum_{j=0}^\infty \amin^{-j} \frac n d = \frac{\amin}{\amin-1}\frac n d \le \amin \frac n d   .
$$

\item $\calT$: a complete multiscale tree containing $\calT_n$. 
$\calT$ can be viewed as the union $\calT_n$ and some empty cells, mostly at fine scales with high probability, that our data have not explored. 

\item $\calT(\tau)$: the smallest subtree of $\calT$ which contains $\{\Cjk \in \calT \,:\,\Deltajkell \ge \tau\}$.

\item $\calT_n(\tau) := \calT(\tau) \cap \calT_n$.

\item $\hcalT_n(\tau)$: the smallest subtree of $\calT_n$ which contains $\{\Cjk \in \calT_n\,:\,\hDeltajkell \ge \tau\}$.

\item $\Lam(\tau)$: the adaptive partition associated with $\calT(\tau)$.

\item $\Lam_n(\tau)$: the adaptive partition associated with $\calT_n(\tau)$.

\item $\hLam_n(\tau)$: the adaptive partition associated with $\hcalT_n(\tau)$.

\item Suppose $\calT^0$ and $\calT^1$ are two subtrees of $\calT$. If $\Lam^0$ and ${\Lam}^1$ are two adaptive partitions associated with $\calT^0$ and $\calT^1$ respectively, we denote by $\Lam^0 \vee \Lam^1$ and $\Lam^0 \wedge \Lam^1$ the partitions associated to the trees $\calT^0 \cup \calT^1$ and $\calT^0 \cap \calT^1$ respectively. 

\item  Let $b=2\amax+5$ where $\amax$ is the maximal number of children that a node has in $\calT_n$. 
\end{itemize}

Inspired by the analysis of wavelet thresholding procedures \citep{BCDDT1,BCDD2}, we split the error into four terms,
$$ \| f - \hf^\ell_{\hLam_n(\tau_n)}  \| \le e_1 + e_2 + e_3 +e_4 , $$
where
$$ \begin{matrix*}[l]
 e_1 &:= \| f - f^\ell_{\hLam_n(\tau_n)\vee \Lam_n(b\tau_n)} \| \\
 e_2 &:= \| f^\ell_{\hLam_n(\tau_n)\vee \Lam_n(b\tau_n)} - f^\ell_{\hLam_n(\tau_n)\wedge \Lam_n(\tau_n/b)} \| \\
 e_3 &:= \| f^\ell_{\hLam_n(\tau_n)\wedge \Lam_n(\tau_n/b)} - \hf^\ell_{\hLam_n(\tau_n)\wedge \Lam_n(\tau_n/b)} \| \\
 e_4 &:= \| \hf^\ell_{\hLam_n(\tau_n)\wedge \Lam_n(\tau_n/b)} - \hf^\ell_{\hLam_n(\tau_n)} \|
\end{matrix*}. $$
The goal of the splitting above is to handle the bias and variance separately, as well as to deal with the fact the partition built from those $\Cjk$ such that $\hDeltajkell \ge \tau_n$ does not coincide with the partition which would be chosen by an oracle based on those $\Cjk$ such that $\Deltajkell \ge \tau_n$. This is accounted by the terms $e_2$ and $e_4$ which correspond to those $\Cjk$ such that $\hDeltajkell$ is significantly larger or smaller than $\Deltajkell$ respectively, and which will be proved to be small in probability. The $e_1$ and $e_3$ terms correspond to the bias and variance of oracle estimators based on partitions obtained by thresholding the unknown oracle change in approximation $\Deltajkell$.

Since  $\hLam_n(\tau_n)\vee \Lam_n(b\tau_n)$ is a finer  partition than $\Lam_n(b\tau_n)$, we have
$$ e_1 \le \| f - f^\ell_{\Lam_n(b\tau_n)} \| \le \| f - f^\ell_{\Lam(b\tau_n)} \|+ \| f^\ell_{\Lam(b\tau_n)} - f^\ell_{\Lam_n(b\tau_n)} \|=: e_{11} + e_{12}\,.$$
The $e_{11}$ term is treated by a deterministic estimate based on the model class $\BSL$: by Lemma \ref{lem:bias-Bs} we have
$$ e_{11}^2 \le B_{s,d} |f|_{\BSL}^p (b\kappa)^{2-p} \lognnnf^{\frac{2s}{2s+d}} , $$
The term $e_{12}$ accounts for the error on the cells that have not been explored by our training data, which is small:
\begin{align*}
 \PP \{ e_{12} > 0 \} &\le \PP \{ \exists \ \Cjk \in \calT(b\tau_n) \setminus \calT_n(b\tau_n) \} \\
                             &= \PP \{ \exists \ \Cjk \in \calT(b\tau_n) : \Deltajkell \ge b\tau_n  \text{ and }\hrho(\Cjk) < d/n \} \\
                             &\le \sum_{\Cjk\in\calT(b\tau_n)} \PP \{\Deltajkell \ge b\tau_n  \text{ and } \hrho(\Cjk) < d/n \} .
\end{align*}
According to \eqref{eq:Delta2<M2rho},  we have $(\Deltajkell)^2 \le 4\|f\|_\infty^2 \rho(\Cjk)$.
Then every $ \Cjk $ with $\Deltajkell \ge b\tau_n$ satisfies $ \rho(\Cjk) \gtrsim \frac{b^2\kappa^2}{\|f\|_\infty^2} \lognnnf $.
Hence, provided that $n$ satisfies $ \frac{b^2\kappa^2}{\|f\|_\infty^2} \log n \gtrsim 2d $, we have
\begin{align*} 
&\PP \{\Deltajkell \ge b\tau_n \text{ and } \hrho(\Cjk) < d/n \} \\
&\le \PP \left\{ |\rho(\Cjk) - \hrho(\Cjk)| \ge \frac{1}{2} \rho(\Cjk) \text{ and } \rho(\Cjk) \gtrsim \frac{b^2\kappa^2}{\|f\|_\infty^2} \lognn \right\} 
\\
& \le 2n^{- \frac{3b^2\kappa^2}{28\|f\|_\infty^2 }},
\end{align*}
where the last inequality follows from Lemma \ref{lem:rhocsigma}\ref{rho-hrho}.
Therefore, by Definition \ref{def:Bs} we obtain
\begin{align*}
 \PP \{ e_{12} > 0 \} &\lesssim \#\calT(b\tau_n) n^{- \frac{3b^2\kappa^2}{28\|f\|_\infty^2}}
  \le |f|_{\BSL}^{p} (b\tau_n)^{-p} n^{- \frac{3b^2\kappa^2}{28\|f\|_\infty^2} }  \\
 & \le |f|_{\BSL}^{p} (b\kappa)^{-p} n^{- \left( \frac{3b^2\kappa^2}{28\|f\|_\infty^2} - 1 \right)}
  \le |f|_{\BSL}^{p} (b\kappa)^{-p} n^{-\nu}
\end{align*}
as long as $\frac{3b^2\kappa^2}{28\|f\|_\infty^2} - 1 >\nu$.
To estimate $\EE e_{12}^2$, we observe that, thanks to Lemma \ref{lem:fjk-fj'k'},
$$
 e_{12}^2
 = \sum_{\Cjk\in\Lam_n(b\tau_n)\setminus\Lam (b\tau_n)} \ \sum_{\substack{C_{j',k'}\in\Lam(b\tau_n) \\ C_{j',k'}\subset\Cjk}} \| (f^\ell_{j,k} - f^\ell_{j',k'}) \mathbf{1}_{j',k'} \|^2 \lesssim M^2
$$
Hence, by choosing $ \nu = 1 > \tfrac{2s}{2s+d} $ we get
$$ \EE e_{12}^2 \lesssim \|f\|_\infty^2 \PP \{e_{12} > 0\} \lesssim \|f\|_\infty^2 |f|_{\BSL}^{p} (b\kappa)^{-p} \lognnnf^{\frac{2s}{2s+d}} . $$

The term $e_3$ is the variance term which can be estimated by Proposition \ref{prop:var-bound} with $ \Lam = \hLam_n(\tau_n)\wedge \Lam_n(\tau_n/b) $.
We plug in $ \eta = r (\log n / n)^{\frac{s}{2s+d}} $.
Bounding $ \#\Lam $ by $ \# \Lam_n(\tau_n/b) \le \#\Lam_n \le n /d $ (as our data master tree has at $d$ points in each leaf) outside the exponential,
and by $ \# \Lam_n(\tau_n/b) \le \# \Lam(\tau_n/b) \le |f|_{\BSL}^{p} (\tau_n/b)^{-p} $ inside the exponential, we get the following estimates for $e_3$:
$$
\PP \left\{ e_3 > r \left( \lognn\right)^{\frac{s}{2s+d}} \right\}
 \le
 \begin{cases}
  C_0 n^{ 1- \frac{ r^2 \kappa^{p}}{c_0 b^{p} |f|_{\BSL}^{p} \max\{M^2,\sigma^2\} }} & \ell = 0 \\
  C_1 n^{ 1 - \frac{ \gamma^2 \kappa^{p}}{c_1 b^{p} |f|_{\BSL}^{p} \max\{d^4M^2,d^2\sigma^2\} }} & \ell = 1 
 \end{cases},
$$
where $C_0 = C_0(\theta_2,\theta_3,\amax,d,s,|f|_{\BSL},\kappa)$ and $C_1 = C_1(\theta_2,\theta_3,\amax,d,s,|f|_{\BSL},\kappa)$. We obtain $\PP\{e_3 > r(\log n /n)^{\frac{s}{2s+d}}\} \le C n^{-\nu}$ as long as $r$ is chosen large enough to make the exponent smaller than $-\nu$.

To estimate $\EE e_3^2$, we apply again Propositions \ref{prop:var-bound} and  with $ \#\Lam \le |f|_{\BSL}^{p} (b/\kappa)^{p} (\log n / n)^{-\frac{d}{2s+d}} $, obtaining
$$ 
\EE e_3^2 \le \bar{C} \lognnnf^{\frac{2s}{2s+d}}\,.
$$

Next we estimate $e_2$ and $e_4$.
Since $ \hcalT_n(\tau_n) \cap \calT_n(\tau_n/b) \subseteq \hcalT_n(\tau_n) \cup \calT_n(b\tau_n) $ and $ \calT_n(b\tau_n) \subseteq \calT_n(\tau_n/b) $,
we have $ e_2 > 0 $ if and only if there is a $ \Cjk \in \calT_n $ such that either $\Cjk$ is in $\hcalT_n(\tau_n)$ but not in $\calT_n(\tau_n/b)$,
or $\Cjk$ is in $\calT_n(b\tau_n)$ but not in $\hcalT_n(\tau_n)$.
This means that either $ \hDeltajkell \ge \tau_n $ but $ \Deltajkell < \tau_n/b $, or $ \Deltajkell \ge b\tau_n $ but $ \hDeltajkell < \tau_n $.
As a consequence,
$$
 \PP \{ e_2 > 0 \} \le \sum_{\Cjk\in\calT_n} \PP \left\{ \hDeltajkell \ge \tau_n \ \text{and} \ \Deltajkell < \tau_n/b \right\}
+ \sum_{\Cjk\in\calT_n} \PP \left\{ \Deltajkell \ge b\tau_n \ \text{and} \ \hDeltajkell < \tau_n \right\} ,
$$
and analogously
$$
 \PP \{ e_4 > 0 \} \le \sum_{\Cjk\in\calT_n} \PP \left\{ \hDeltajkell \ge \tau_n \ \text{and} \ \Deltajkell < \tau_n/b \right\} .
$$
We can now apply Lemma \ref{lem:Delta-hDelta}: we use \ref{Delta<hDelta>} with $ \eta = \tau_n/b $,
and \ref{Delta>hDelta<} with $ \eta = \tau_n $. We obtain that
\beqn
\PP\{e_2 > 0 \} + \PP\{e_4 >0\}
\le 
\begin{cases}
C(\amin,d) n^{1-\frac{\kappa^2}{c_0 b^2 \max\{M^2,\sigma^2\}}}
& \ell = 0
\\
C(\theta_2,\theta_3,\amin,d) n^{1- \frac{\kappa^2}{c_1 b^2 \max\{ d^4 M^2, d^2\sigma^2\}}}
& \ell =1
\end{cases} .
\eeqn

We have $\PP\{e_2 >0 \} + \PP\{ e_4 > 0\} \le C n^{-\nu}$ provided that $\kappa$ is chosen such that the exponents are smaller than $-\nu$.

We are left to deal with the expectations.
As for $e_{2}$, Lemma \ref{lem:fjk-fj'k'} implies $ e_2 \lesssim M $, which gives rise to, for $ \nu = 1 > \frac{2s}{2s+d}$,
$$ \EE e_2^2 \lesssim M^2 \PP\{e_2 > 0\} \le C M^2 \lognnnf^{\frac{2s}{2s+d}}\,. $$

The same bound holds for $e_4$,
which concludes the proof of Theorem \ref{thm:unif-bound}. \hfill $\blacksquare$

\subsection{Basic concentration inequalities}

This section contains the main concentration inequalities of the empirical quantities on their oracles. For piecewise linear estimators, some quantities used in Lemma \ref{lem:fjk} are decomposed in Table \ref{tab:qr}.
All proofs are collected in Appendix \ref{sec:proofs_lemmata}.

\begin{lemma} \label{lem:rhocsigma}
 For every $ t > 0 $ we have:
 \begin{enumerate}[label=\textnormal{(\alph*)}]
 \item \label{rho-hrho'} $\PP\left\{\left|\rho(\Cjk)-\hrho(\Cjk)\right| > t \right\} \le 2\exp\left( \frac{-3nt^2}{6\rho(\Cjk) +2t}\right) $;
 \item \label{rho-hrho} Setting $ t = \frac 1 2 \rho(\Cjk) $ in \ref{rho-hrho'} yields
 \item[] $\PP\left\{|\rho(\Cjk)-\hrho(\Cjk)| > \frac{1}{2}\rho(\Cjk) \right\} \le 2\exp\left(-\frac{3}{28}n\rho(\Cjk)\right) $;
 \item \label{c-hc} $ \PP\left\{\|\cjk-\hcjk\| > t \right\} \le 2\exp\left(-\frac{3}{28}n\rho(\Cjk)\right) 
 + 8\exp\left(-\frac{3 n\rho(\Cjk) t^2}{12\theta_2^2 2^{-2j}+4\theta_2 2^{-j}t} \right) $;
\item \label{sigma-hsigma}
$
\PP \{\|\Sjk-\hSjk\| > t \}
\le 
2\exp\left(-\frac{3}{28}n\rho(\Cjk)\right)
+
\left(\frac{4\theta_2^2 }{\theta_3}d+8\right)
\exp\left(
\frac{-3 n\rho(\Cjk) t^2}{96\theta_2^4 2^{-4j} +16\theta_2^2 2^{-2j}t}
\right).
$
 \end{enumerate}
\end{lemma}

\begin{lemma} \label{lem:Qr}
 We have:
 \begin{enumerate}[label=\textnormal{(\alph*)}]
 \item \label{Q-hQ}
 $
  \PP \{\|\Qjk - \hQjk\| > \frac{48}{\theta_3^2}d^2 2^{4j} \| \Sjk - \hSjk  \|\} \\
  \le  2\exp\left(-\frac{3}{28}n\rho(\Cjk)\right)
   + \left(4\frac{\theta_2^2 }{\theta_3}d+10\right) \exp\left(- \frac{n \rho(\Cjk)}{512 (\theta_2^2/\theta_3)^2 d^2 + \frac{64}{3} (\theta_2^2/\theta_3)d}\right)
 $ .
 \item \label{hQ} $
  \PP \{\|\widehat{Q}^{j,k}\| > \frac{2}{\theta_3} d\ 2^{2j}\}
  \le 2\exp\left(-\frac{3}{28}n\rho(\Cjk)\right) + \left(4\frac{\theta_2^2 }{\theta_3}d+10\right) \exp\left(-\tfrac{n\rho(\Cjk)}{128(\theta_2^2/\theta_3)^2d^2 + \frac{32}{3}(\theta_2^2/\theta_3)d}\right) .
$
\item \label{r-hr}
Suppose $f$ is in $L^\infty$. For every $ t > 0 $, we have\\
$ \begin{aligned}[t]
  &\PP\left\{\
  \|\rjk -\hrjk\| > t
  \right\} \begin{aligned}[t] &\le
  2\exp\left(-\tfrac{3}{28}n\rho(\Cjk)\right) + 8 \exp\left(\tfrac{- n\rho(\Cjk) t^2}{4 \langle\theta_2\rangle^2 \|f\|_\infty^22^{-2j}+2 \langle\theta_2\rangle \|f\|_\infty 2^{-j} t} \right) \\
   &+ 2 \exp\left( - c\tfrac{n\rho(\Cjk)t^2}{\theta_2^2 \|\zeta\|_{\psi_2}^2 2^{-2j}} \right) \end{aligned} \\
  \end{aligned} $ \\
where $c$ is an absolute constant.
 \end{enumerate}
\end{lemma}

\begin{lemma} \label{lem:fjk}
Suppose $f$ is in $L^\infty$. For every $ t > 0 $, we have\\
 $ \PP \left\{ \| f_{j,k}^\ell - \hf_{j,k}^\ell \|_\infty > t \right\} 
   \le \begin{cases}
         C_0 \left[ \exp\left(-\frac{n\rho(\Cjk)}{c_0}\right)
   + \exp\left(- \frac{n\rho(\Cjk) t^2}{c_0 (\|f\|_\infty^2 + \|f\|_\infty t)}\right)
   + \exp\left(-\tfrac{n\rho(\Cjk)t^2}{c_0\|\zeta\|_{\psi_2}^2}\right) \right] & \text{$ \ell = 0 $} \\
         C_1 d \left[ \exp\left( - \frac{n\rho(\Cjk)}{c_1d^2} \right) + \exp\left(- \frac{n\rho(\Cjk)t^2}{c_1 d^4 (\|f\|_\infty^2 + \|f\|_\infty t)}\right) + \exp\left(- \frac{n\rho(\Cjk)t^2}{c_1 d^2 \|\zeta\|_{\psi_2}^2}\right) \right] & \text{$ \ell = 1 $;}
        \end{cases} $ \\
   where $c_0,C_0$ are absolute constants, $c_0'$ depends on $\theta_2$, and $c_1,C_1$ depend on $\theta_2,\theta_3$.
\end{lemma}

\begin{lemma} \label{lem:fjk-fj'k'}
 Suppose $ f \in L^\infty $. For every $ \Cjk \in \calT $ and $ C_{j',k'} \subset \Cjk $,
$$
 \| \fjk - f_{j',k'} \|_\infty \le 2 M , \qquad \| \hfjk - \widehat{f}_{j',k'} \|_\infty \le 2M .
$$
\end{lemma}

\begin{table}[h]
\renewcommand{\arraystretch}{2}
\begin{center}
\resizebox{1\columnwidth}{!}{%
\begin{tabular}{c | c |}
 \hline
 \multicolumn{1}{|c|}{oracle estimators} &  {empirical counterparts }
 \\ \hline
 \multicolumn{1}{|c|}{$\fjk(x) = T_M \left([ (x-\cjk)^T \ \  2^{-j}] \Qjk \rjk\right)$} & $\hfjk(x) =T_M\left([ (x-\hcjk)^T  \ \ 2^{-j} ] \hQjk \hrjk \right)$
  \\ [5pt] \cdashline{1-2}
 \multicolumn{1}{|c|}{$\Qjk := \left[\begin{smallmatrix} [\Sjk]_d^{\dag} &  0 \\ 0 & 2^{2j}\end{smallmatrix}\right]$} & $ \hQjk := \left[\begin{smallmatrix}[\hSjk]_d^{\dag} &  0\\0 & 2^{2j}\end{smallmatrix}\right]$
 \\ [5pt]
 \multicolumn{1}{|c|}{$ [\Sjk]_d^{\dag} := \Vjk [\Lamjk]_d^{-1} \Vjk^T $} & $ [\hSjk]_d^{\dag} := \hVjk {[\hLamjk]_d}^{-1} (\hVjk)^T $
 \\ [5pt]
 \multicolumn{1}{|c|}{$\rjk:=\frac{1}{\rho(\Cjk)} \displaystyle\int_{\Cjk} y \left[\begin{smallmatrix} (x-\cjk) \\ 2^{-j}\end{smallmatrix}\right]d\rho$} & $\hrjk:= \frac{1}{\hnjk} \displaystyle\sum_{x_i \in \Cjk} y_i \left[\begin{smallmatrix} (x_i-\hcjk)\\2^{-j}\end{smallmatrix}\right]$
  \\ [5pt] \hline
 \multicolumn{1}{|c|}{$\fjk(x) = \yjk + T_{L(\theta_2 2^{-j})^\beta} \left([ (x-\cjk)^T \ \  2^{-j}] \Qjk \rjk \right)$} & $\hfjk = \hyjk + T_{L(\theta_2 2^{-j})^\beta}\left([ (x-\hcjk)^T  \ \ 2^{-j} ] \hQjk \hrjk \right)$
  \\ \cdashline{1-2}
  \multicolumn{1}{|c|}{$\rjk:=\frac{1}{\rho(\Cjk)} \displaystyle\int_{\Cjk} (y-\yjk) \left[\begin{smallmatrix} (x-\cjk)\\2^{-j}\end{smallmatrix}\right]d\rho$}
  & $\hrjk:= \frac{1}{\hnjk} \displaystyle\sum_{x_i \in \Cjk} (y_i - \hyjk) \left[\begin{smallmatrix}(x_i-\hcjk)\\2^{-j}\end{smallmatrix}\right]$
  \\ \hline
\end{tabular}%
}
\caption{
Decomposition of piecewise linear estimators into quantities used in Lemma \ref{lem:fjk}.
}
\label{tab:qr}
\end{center}
\end{table}

\begin{lemma} \label{lem:Delta-hDelta}
Suppose $f$ is in $L^\infty$.
For every $ \eta > 0 $ and any $ \gamma > 1 $, we have
 \begin{enumerate}[label=\textnormal{(\alph*)},leftmargin=*]
 \item \label{Delta>hDelta<}
$
\PP\left\{\hDeltajkell < \eta \ \& \
\Deltajkell \ge (2\amax+5)\eta  \right\}
 \le
 \begin{cases}
  C_0 \exp\left( - \tfrac{n\eta^2}{c_0 \max\{\|f\|_\infty^2,\|\zeta\|_{\psi_2}^2\}} \right) & \ell = 0 \\
 C_1 d \exp\left( - \tfrac{n\eta^2}{c_1 \max\{d^4\|f\|_\infty^2,d^2\|\zeta\|_{\psi_2}^2\}} \right) & \ell = 1
 \end{cases} $
\item \label{Delta<hDelta>}
$ \PP\left\{\Deltajkell < \eta \ \& \
\hDeltajkell \ge (2\amax+5)\eta  \right\}
 \le \begin{cases}
  C_0 \exp\left( - \tfrac{n\eta^2}{c_0 \max\{\|f\|_\infty^2,\|\zeta\|_{\psi_2}^2\}} \right) & \ell = 0 \\
 C_1 d \exp\left( - \tfrac{n\eta^2}{c_1 \max\{d^4\|f\|_\infty^2,d^2\|\zeta\|_{\psi_2}^2\}} \right) & \ell = 1 ;
 \end{cases} $
\end{enumerate}
$C_0,c_0$ depend on $\amax$; $c_0'$ depends on $\amax,\theta_2$; $C_1,c_1$ depend on $\amax,\theta_2,\theta_3$.
\end{lemma}

\section{Acknowledgements}
This work was partially supported by NSF-DMS-125012, AFOSR FA9550-17-1-0280, NSF-IIS-1546392. The authors are grateful to Duke University for donating computing equipment used for this project.

\bibliography{Ref}

\begin{appendices}

\section{Additional proofs} \label{sec:proofs_lemmata}

\begin{proof}[Example \ref{hold-As}]
 Let $ f \in \calC^{\ell,\alpha} $.
 The local estimator $\fjkell$ minimizes $ \|(f - p)\chijk\| $ over all possible polynomials $p$ of order less than or equal to $\ell$.
 Thus, in particular, we have $ \|(f - \fjkell)\chijk\| \le \| (f - p)
\chijk \| $
 where $p$ is equal to the $\ell$-order Taylor polynomial of $f$ at some $ z \in \Cjk $.
 Hence, for $ x \in \Cjk $ there is $ \xi \in \calM \cap B_{\theta_2 2^{-j}}(z) $ such that
 \begin{align*}
 | f(x) - p(x) | 
&\le \sum_{|\la|=\ell} \frac{1}{\la!} | \partial^\la f(\xi) - \partial^{\la} f(z) | | x - z |^\la \\
 &\le |f|_{\calC^{\ell,\alpha}} \| \xi - z \|^\alpha \sum_{|\la|=\ell} \frac{1}{\la!} | x - z |^\la \\
& \le \frac{d^\ell}{\ell!} |f|_{\calC^{\ell,\alpha}} \| \xi - z \|^\alpha \| x - z \|^\ell 
\\
& \le \theta_2^{\ell+\alpha} \frac{d^\ell}{\ell!} |f|_{\calC^{\ell,\alpha}} 2^{-j(\ell+\alpha)} .
 \end{align*}
 Therefore, for every $j $ and $k \in \calK_j$, we have 
 \begin{equation*} \|(f - \fjkell)\chijk\|^2 \le \theta_2^{2(\ell+\alpha) } (\frac{d^\ell}{\ell!})^2 |f|_{\calC^{\ell,\alpha}}^2 2^{-2j(\ell+\alpha)} \rho(\Cjk) . \end{equation*}
 \end{proof}

\begin{proof}[Examples \ref{example1} and \ref{example1a}]
For polynomial estimators of any fixed order $\ell = 0,1,\ldots$, $g^\ell_{j,k} - g\chijk = 0$ when $\Cjk \cap \Gamma =\emptyset$, and  $g^\ell_{j,k} - g\chijk = {O}(1)$ when $\Cjk \cap \Gamma \neq \emptyset$. At the scale $j$, $\rho(\Cjk) \approx 2^{-jd}$ and $\rho(\cup\{\Cjk: \Cjk\cap \Gamma \neq \emptyset\}) \approx 2^{-j(d-d_\Gamma)}\rho(\Gamma)$. Therefore,
$$\|g^\ell_{\Lam_j}-g\| \le {O}(\sqrt{2^{-j(d-d_\Gamma)}}) = {O}(2^{-j (d-d_\Gamma)/{2}}),$$
which implies $g\in \calA^{\ell}_{(d-d_\Gamma)/{2}}$.

In adaptive approximations, $\Deltajkell = 0$ when $\Cjk \cap \Gamma =\emptyset$.  When $\Cjk \cap \Gamma \neq \emptyset$, 
  $\Deltajkell =\|g^{\ell}_{j,k} - \sum_{C_{j+1,k'}\subset\Cjk}g^{\ell}_{j+1,k'}\| \lesssim \sqrt{\rho(\Cjk)} \lesssim 2^{-jd/2}$. Given any fixed threshold $\tau>0$, in the truncated tree $\calT(\tau)$, the leaf nodes intersecting with $\Gamma$ satisfy $2^{-jd/2} \gtrsim \tau$. In other words, around $\Gamma$ the tree is truncated at a coarser scale than $j^\star$ such that $2^{-j^\star} ={O}(\tau^{\frac 2 d})$. The cardinality of $\calT(\tau)$ is dominated by the nodes intersecting with $\Gamma$, so
  $$\#\calT(\tau) \lesssim \frac{\rho(\Gamma) 2^{-j^\star(d-d_\Gamma)}}{2^{-j^\star d}} = \rho(\Gamma)2^{j^\star d_\Gamma} \lesssim \tau^{-\frac{2d_\Gamma}{d}},$$
  which implies $p = 2d_\Gamma/d$. We conclude that $g\in \calB^{\ell}_s$ with $s = \frac{d(2-p)}{2p} = \frac{d}{d_\Gamma} (d-d_\Gamma)/2.$
\end{proof}

 \begin{proof}[Lemma \ref{AsInfBs}]
By definition, we have $\|(f-\fjkell)\chijk\| \le |f|_{\ASLINF} 2^{-js} \sqrt{\rho(\Cjk)}$ as long as $f \in \ASLINF$. By splitting $(\Delta^\ell_{j,k})^2 \le 2\|(f - \fjkell)\chijk\|^2 + 2\sum_{k':C_{j+1,k'}\subset\Cjk} \|(f - f^\ell_{j+1,k'})\mathbf{1}_{j+1,k'}\|^2 $, we get
$$
  (\Delta^\ell_{j,k})^2 \le 4 |f|^2_{\ASLINF} 2^{-2js}\rho(\Cjk).
$$
In the selection of adaptive partitions, every $ \Cjk $ with $ \Deltajkell \ge \tau $ must satisfy
 $ \rho(\Cjk) \ge 2^{2js} (\tau/|f|_{\ASLINF})^2 $. 
 With extra assumptions $ \rho(\Cjk) \le \theta_0 2^{-jd} $ (true when the measure $\rho$ is doubling), we have
\begin{equation}  \label{eq:coarser}
\Deltajkell \ge \tau \Longrightarrow 
 2^{-j} \ge  \left(\frac{\tau}{|f|_{\ASLINF}}\right)^{\frac{2}{2s+d}}.
 \end{equation}
Therefore, every cell in $\Lam(\tau)$ will be at a coarser scale than $j^\star$ with $j^\star$ satisfying \eqref{eq:coarser}.
Using \ref{A3} we thus get
\begin{align*}
 &\tau^{p} \# \calT(\tau) \le \tau^{p} \amin \#\Lam_{j^\star}
  \le \theta_1^{-1} \tau^{p}  \amin 2^{j^\star d} 
 \le \frac{\amin  |f|_{\ASLINF}^{\frac{2d}{2s+d}}}{\theta_1}
\end{align*}
which yields Lemma \ref{AsInfBs}.
\end{proof}

\begin{proof}[Lemma \ref{lem:bias-Bs}]
 For any partition $ \Lam \subset \calT $, denote by $\Lam^{l}$ the $l$-th generation partition such that $\Lam^0 = \Lam$ and $\Lam^{l+1} $ consists of the children of $\Lam^l$.
We first prove that $ \lim_{l\to\infty} f^\ell_{\Lam^l} = f $ in $ L^2(\calM) $.
Suppose $ f \in L^\infty $.
Notice that $ \| f^\ell_{\Lam^l} - f \| \le \| f_{\Lam^l}^0 - f \| $. 
As a result of the Lebesgue differentiation theorem, $ f_{\Lam^l}^0 \to f $ almost everywhere.
Since $f$ is bounded, $f_{\Lam^l}^0$ is uniformly bounded, hence $ f_{\Lam^l}^0 \to f $ in $L^2(\calM)$ by the  dominated convergence theorem.
In the case where $ f \in \calA_t^\ell $, taking the uniform partition $\Lam_{j(l)}$ at the coarsest scale of $\Lam^l$, denoted by $j(l)$, we have
$  \|f - f^\ell_{\Lam^l}\| \le \|f - f^\ell_{\Lam_{j(l)}} \| \lesssim 2^{-j(l)t} $, and therefore $ f^\ell_{\Lam^l} \to f $ in $L^2(\calM)$.

Now, setting $ \Lam = \Lam(\tau) $ and $ \calS := \calT(\tau)\setminus\Lam $, by Definitions \ref{def:Bs} and \ref{def:quasi-ortho} we get
\begin{align*}
 \| f^\ell_{\Lam} - f \|^2 
&= \left\| \sum_{l=0}^{L-1} \left(f^\ell_{\Lam^l} - f^\ell_{\Lam^{l+1}}\right)  + f^\ell_{\Lam^L} - f \right\|^2
 = \left\| \sum_{l=0}^\infty (f^\ell_{\Lam^l} - f^\ell_{\Lam^{l+1}}) \right\|^2 \\
 &= \left\| \sum_{\Cjk\in\calT\setminus\calS} \Wjk^\ell \right\|^2 
\le B_0 \sum_{\Cjk \in \calT\setminus \calS} \|\Wjk^\ell\|^2 \\
 &= B_0 \sum_{l=0}^\infty \sum_{\Deltajkell \in [2^{-(l+1)}\tau,2^{-l}\tau)} \Delta_{j,k}^2 
 \le B_0 \sum_{l=0}^\infty 2^{-2 l} \tau^2 \#\calT(2^{-(l+1)}\tau) \\
 &= B_0 2^p \tau^{2-p} \sum_{l=0}^\infty 2^{-(2-p)l}  |f|_{\BSL}^p 
 \le B_0 2^p |f|_{\BSL}^p \tau^{2-p} \sum_{l=0}^\infty 2^{-(2-p) l} ,
\end{align*}
which yields the first inequality in Lemma \ref{lem:bias-Bs}.
The second inequality follows by observing that $ 2-p = \frac{2s}{d} p $ and
$ |f|_{\BSL}^p \tau^{2-p} = |f|_{\BSL}^2 (|f|_{\BSL}^{-p} \tau^p)^{\frac{2s}{d}} \le |f|_{\BSL}^2 \#\left[\calT(\tau)\right]^{-\frac{2s}{d}} $ by Definition \ref{def:Bs}.
\end{proof}

\begin{proof}[Lemma \ref{lem:rhocsigma}]
See \cite{LiaoMaggioni}.
\end{proof}

\begin{proof}[Lemma \ref{lem:Qr}]
 \ref{Q-hQ}.
 Thanks to \cite[Theorem 3.2]{TSVD} and assumption \ref{A5}, we have
\begin{align*}
 \| \Qjk - \hQjk \| &= \| [\Sjk]_d^{\dag} - [\hSjk]_d^{\dag} \| \le 3 \tfrac{\| \Sjk - \hSjk \|}{(\lamjk_d - \lamjk_{d+1} - \|\Sjk - \hSjk\|)^2}  
  \le 3\tfrac{\| \Sjk - \hSjk \|}{\left( \frac{\theta_3}{2d}2^{-2j} - \|\Sjk - \hSjk\|\right)^2} .
\end{align*}
Hence, the bound follows applying Lemma \ref{lem:rhocsigma}\ref{sigma-hsigma} with $ t = \frac{\theta_3}{4d}2^{-2j} $.

\ref{hQ}.
Observe that
$ \|\hQjk\| \le \|[\hSjk]_d^{\dag}\| = (\hlamjk_d)^{-1}$.
Moreover,
$ \hlamjk_d \ge \lamjk_d - |\lamjk_d - \hlamjk_d| \ge \frac{\theta_3}{d}2^{-2j} - \|\Sjk - \hSjk\|$
by assumption \ref{A5}.
Thus, using Lemma \ref{lem:rhocsigma}\ref{sigma-hsigma} with $ t = \frac{\theta_3}{2d}2^{-2j} $ yields the result.

\ref{r-hr}.
We condition on the event that $ \hnjk \ge \frac{1}{2}\EE\hnjk = \frac{1}{2}n\rho(\Cjk) $,
whose complement occurs with probability lower than $ 2\exp\left(-\frac{3}{28}n\rho(\Cjk)\right) $ by Lemma \ref{lem:rhocsigma}\ref{rho-hrho}.
The quantity $ \| \rjk - \hrjk \| $ is bounded by $ A + B + C + D $ with
\begin{align*}
 &A := \left\| \frac{1}{\hnjk} \sum_{i=1}^{n} \left( f(x_i) \begin{bmatrix} x_i - \cjk \\ 2^{-j} \end{bmatrix} - \frac{1}{\rho(\Cjk)} \int_{\Cjk} f(x) \begin{bmatrix} x - \cjk \\ 2^{-j} \end{bmatrix} d\rho \right) \chijk(x_i) \right\| \\
 &B := \left\| \frac{1}{\hnjk} \sum_{i=1}^{n} f(x_i) \begin{bmatrix} \cjk - \hcjk \\ 0 \end{bmatrix} \chijk(x_i) \right\| \\
 &C := \left\| \frac{1}{\hnjk} \sum_{i=1}^{n} \zeta_i \begin{bmatrix} x_i - \cjk \\ 2^{-j} \end{bmatrix} \chijk(x_i) \right\| \\
 &D := \left\| \frac{1}{\hnjk} \sum_{i=1}^{n} \zeta_i \begin{bmatrix} \cjk - \hcjk \\ 0 \end{bmatrix} \chijk(x_i) \right\| .
\end{align*}
Each term of the sum in $A$ has expectation $0$ and bound $2\langle\theta_2\rangle \|f\|_\infty 2^{-j}$.
Thus, applying the Bernstein inequality \cite[Corollary 7.3.2]{TroppNIPS} we obtain
$$ \PP \{ A > t \} \le 8 \exp\left( - c \tfrac{n\rho(\Cjk)t^2}{\langle\theta_2\rangle^2 \|f\|_\infty^2 2^{-2j} + \langle\theta_2\rangle \|f\|_\infty 2^{-j} t} \right) . $$
$B$ is bounded by $ \|f\|_\infty \|\cjk -  \hcjk\| $
so that, using \ref{lem:rhocsigma}\ref{c-hc} with $t$ replaced by $ t / \|f\|_\infty $, we get
$$ \PP \{ B > t \} \le 2\exp\left(-\tfrac{3}{28}n\rho(\Cjk)\right) + 8\exp\left(-c\tfrac{n\rho(\Cjk) t^2}{\theta_2^2 \|f\|_\infty^2 2^{-2j} + \theta_2 \|f\|_\infty 2^{-j}t} \right) . $$

To estimate $C$ we appeal to \citet[Theorem 3.1, Remark 4.2]{Buldygin-Pechuk}.
For $ X \in \RR^n $, take $ G(X) := \|MX\| $ with $ M := \left[\begin{smallmatrix} x_1 - \cjk \\ 2^{-j} \end{smallmatrix} \dots \begin{smallmatrix} x_n - \cjk \\ 2^{-j} \end{smallmatrix} \right] $.
Then $ |\partial_i G(X)| \le \|x_i - \cjk\| \le \theta_2 2^{-j} $. Now let $ X = (\zeta_1 \chijk(x_1),\dots,\zeta_n \chijk(x_n))^T $, so that $ C = G(X) / \hnjk $.
Since the $\zeta_i$'s are independent, \citet[Remark 4.2]{Buldygin-Pechuk} applies, and it yields $ \PP \left\{ G(X) > t \right\} \le 2 \exp\left( - \frac{t^2}{2 \sigma^2} \right) $,
where $ \sigma^2 = \sum_{i=1}^n \|\partial_iG\|_\infty^2 \|\zeta_i\|_{\psi_2}^2 \chijk(x_i) \le \hnjk \theta_2^2 2^{-2j} \|\zeta\|_{\psi_2}^2 $, and thus
$$ \PP \left\{ C > t \right\} \le 2 \exp\left( - \tfrac{n\rho(\Cjk)t^2}{2 \theta_2^2 \|\zeta\|_{\psi_2}^2 2^{-2j}} \right) . $$

We are left with $D$. This term is smaller than $ \|\cjk - \hcjk\| \left| \frac{1}{\hnjk} \sum_{i=1}^{n} \zeta_i \chijk(x_i) \right| $,
where, by Lemma \ref{lem:rhocsigma}\ref{c-hc}, $ \|\cjk - \hcjk\| \le \theta_2 2^{-j} $ with probability higher than $ 1 - 10 \exp\left( - \frac{3}{28} n \rho(\Cjk) \right) $.
Hence, by the standard sub-Gaussian tail inequality \cite[Proposition 5.10]{Vershynin:NARMT} we have
$$ \PP \{ D > t \} \le 10 \exp\left( - \tfrac{3}{28} n \rho(\Cjk) \right) + e \exp\left( - c \tfrac{n\rho(\Cjk)t^2}{\theta_2^2 \|\zeta\|_{\psi_2}^2 2^{-2j}} \right) . $$
This completes the proof.
\end{proof}

\begin{proof}[Lemma \ref{lem:fjk}]
If $ \ell = 0 $, then $ \| \hfjkell - \hfjkell \|_\infty = | \yjk - \hyjk | $, which is less than
$$
 \left| \frac{1}{\hnjk}\sum_{i=1}^{n} \left( f(x_i) - \frac{1}{\rho(\Cjk)} \int_{\Cjk} f(x) d\rho(x) \right) \chijk(x_i) \right| +
 \left| \frac{1}{\hnjk}\sum_{i=1}^{n} \zeta_i \chijk(x_i) \right| .
$$
Each addend in the first term has expectation $0$
 and bound $ 2 \|f\|_\infty $,
and therefore we can apply the standard Bernstein inequality \cite[Theorem 1.6.1]{TroppNIPS}.
As for the second term, we use the standard sub-Gaussian tail inequality \cite[Proposition 5.10]{Vershynin:NARMT}.
This yields the bounds for $ \ell = 0 $.

For $ \ell = 1 $, we have
\begin{align*}
&|\fjk(x)-\hfjk(x)|
\le  |[(x-\cjk)^T \ 2^{-j}]^T \Qjk \rjk
- [(x-\hcjk)^T \ 2^{-j}]^T \hQjk \hrjk |
\\
\le \ &| \begin{bmatrix}(\cjk-\hcjk)^T & 0 \end{bmatrix} \Qjk\rjk | + | \begin{bmatrix}(x-\hcjk)^T & 2^{-j} \end{bmatrix} (\Qjk\rjk - \hQjk\hrjk) | \\
                                     \le \ &\| \cjk - \hcjk \| \ \|\Qjk\| \ \|\rjk\| + \|\begin{bmatrix}(x-\hcjk)^T & 2^{-j}\end{bmatrix}\| \ \|\Qjk\rjk - \hQjk\hrjk \| \\
                                     \le \ &\| \cjk - \hcjk \| \ \|\Qjk\| \ \|\rjk\| + \|\begin{bmatrix}(x-\hcjk)^T & 2^{-j}\end{bmatrix}\| \ \left( \| \Qjk - \hQjk \| \ \|\rjk\| + \|\hQjk\| \|\rjk-\hrjk\| \right) \\
                                     \lesssim \ &\tfrac{\theta_2^2}{\theta_3} \left( d \|f\|_\infty 2^j \| \cjk - \hcjk \| + d^2 M 2^{2j} \| \Sjk - \hSjk \| + d 2^{j} \|\rjk-\hrjk\|\right),
 \end{align*}
where the last inequality holds
with high probability thanks to Lemma \ref{lem:Qr}\ref{Q-hQ}\ref{hQ}.
Thus, applying Lemma \ref{lem:rhocsigma}\ref{c-hc}\ref{sigma-hsigma} and Lemma \ref{lem:Qr}\ref{r-hr} with $t$ replaced by $ \frac{t}{\theta d M2^j} $,
$ \frac{t}{\theta d^2 M2^{2j}} $ and $ \frac{t}{\theta d 2^j} $,
we obtain the desired result.
\end{proof}

\begin{proof}[Lemma \ref{lem:fjk-fj'k'}]
Follows simply by truncation.
\end{proof}

\begin{proof}[Lemma \ref{lem:Delta-hDelta}]
 We start with \ref{Delta>hDelta<}. Defining $ \bDeltajkell := \| \Wjk^\ell \|_n $ we have
\begin{align*}
 &\PP \left\{ \hDeltajkell < \eta \ \text{ and } \ \Deltajkell \ge (2\amax+5) \eta  \right\} \\
 \le \ &\PP \left\{ \hDeltajkell < \eta \ \text{ and } \ \bDeltajkell \ge (\amax+2)\eta  \right\} + \PP \left\{ \bDeltajkell < (\amax+2)\eta \ \text{ and } \ \Deltajkell \ge (2\amax+5)\eta  \right\} \\
 \le \ &\PP \left\{ | \bDeltajkell - \hDeltajkell | \ge (1+\amax)\eta  \right\} + \PP \left\{ | \Deltajkell - 2\bDeltajkell | \ge \eta  \right\} .
\end{align*}
The first quantity can be bounded by
\begin{align*}
 | \bDeltajkell - \hDeltajkell | &\le \| \Wjk^\ell - \hWjk^\ell \|_n \le \| \fjkell - \hfjkell \|_n + \sum_{C_{j+1,k'}\subset\Cjk} \| f^\ell_{j+1,k'} - \hf^\ell_{j+1,k'} \|_n \\
 &\le \|\fjkell - \hfjkell\|_\infty \sqrt{\hrho(\Cjk)} + \sum_{C_{j+1,k'}\subset\Cjk} \|f^\ell_{j+1,k'} - \hf^\ell_{j+1,k'}\|_\infty \sqrt{\hrho(C_{j+1,k'})} ,
\end{align*}
so that
\begin{align*}
      &\PP \left\{ | \bDeltajkell - \hDeltajkell | \ge (1+\amax)\eta  \right\} \\
 \le \ &\PP \left\{ \| \fjkell - \hfjkell \|_\infty \sqrt{\hrho(\Cjk)} \ge \eta \right\} + \sum_{C^{j+1,k'}\subset\Cjk} \PP \left\{ \| f^\ell_{j+1,k'} - \hf^\ell_{j+1,k'}\|_\infty \sqrt{\hrho(C_{j+1,k'})} \ge \eta \right\} .
\end{align*}
We now condition on the event that $ |\rho(\Cjk) - \hrho(\Cjk)| \le \frac{1}{2}\rho(\Cjk) $, which entails $ \hrho(\Cjk) \le \frac{3}{2} \rho(\Cjk) $,
and apply Lemma \ref{lem:fjk} with $ t \lesssim \eta/\sqrt{\rho(Cjk)} $.
The probability of the complementary event is bounded by Lemma \ref{lem:rhocsigma}\ref{rho-hrho}.
To get rid of the remaining $\rho(\Cjk)$'s inside the exponentials, we lower bound $\rho(\Cjk)$ as follows.
We have
\begin{equation} \label{eq:Delta2<M2rho}
 (\Deltajkell)^2 \le 4\|f\|_\infty^2 \rho(\Cjk) ,
\end{equation}
Thus, $ \Deltajkell \ge (2a+5)\eta $ implies $ \rho(\Cjk) \ge \frac{(2a+5)^2\eta^2}{4\|f\|_\infty^2} $.
Therefore, we obtain that
\begin{align*}
 \PP \{ | \bDeltajkell - \hDeltajkell | \ge (1+\amax)\eta \}  \le \begin{cases}
  C_0 \exp\left( - \tfrac{n\eta^2}{c_0 \max\{\|f\|_\infty^2,\|\zeta\|_{\psi_2}^2\}} \right) & \ell = 0 \\
  C_1 d \exp\left( - \tfrac{n\eta^2}{c_1 \max\{d^4\|f\|_\infty^2,d^2\|\zeta\|_{\psi_2}^2\}} \right) & \ell = 1 ,
 \end{cases} 
 \end{align*}
where $ C_0,c_0 $ depend on $a$,  and $ C_1,c_1 $ depend on $\amax,\theta_2,\theta_3$.

Next we estimate $ \PP \left\{ \Deltajkell - 2\bDeltajkell \ge \eta  \right\} $ by \citet[Theorem 11.2]{GKKW}.
Notice that for all $x\in \calM$, $ |\Wjk^\ell(x)| \lesssim \|f\|_\infty $.
If $ x \notin \Cjk $, then $ \Wjk(x) = 0 $, otherwise there is $k'$ such that $ x \in \calC_{j+1,k'} \subset \Cjk $. In such a case, $ |\Wjk(x)| = | \fjk(x) - f_{j+1,k'}(x) | $, and the claim follows from Lemma \ref{lem:fjk-fj'k'}.
Thus, \cite[Theorem 11.2]{GKKW} gives us
$$ \PP \left\{ \Deltajkell - 2\bDeltajkell \ge \eta  \right\} \lesssim \exp\left( - \tfrac{n\eta^2}{c \|f\|_\infty^2 } \right) , $$
where $c$ is an absolute constant.

Let us turn to \ref{Delta<hDelta>}.
We first observe that
\begin{equation}  \label{eq:hDelta<Msqrt(hrho)}
 \hDeltajkell \lesssim M \sqrt{\hrho(\Cjk)} .
\end{equation}
To see this, note again that $ \hWjk(x) \ne 0 $ only when $ x \in \calC_{j+1,k'} \subset \Cjk $ for some $k'$,
in which case $ |\hWjk(x)| = | \hfjk(x) - \hf_{j+1,k'}(x) | $ and we can apply Lemma \ref{lem:fjk-fj'k'}.
Now note that $ b = 2\amax+5 $. We have
\begin{align*}
 \PP \left\{\Deltajkell < \eta \ \text{ and } \ \hDeltajkell \ge b\eta  \right\} &\le \PP \left\{\Deltajkell < \eta , \ \hDeltajkell \ge b\eta \ \text{ and } \ \rho(\Cjk) \ge \frac{b^2\eta^2}{2\|f\|_\infty^2} \right\} \\
 &+ \PP \left\{ \rho(\Cjk) < \frac{b^2\eta^2}{2\|f\|_\infty^2} \ \text{ and} \ \hrho(\Cjk) \ge \frac{b^2\eta^2}{\|f\|_\infty^2} \right\} \\
 &+ \PP \left\{ \hrho(\Cjk) < \frac{b^2\eta^2}{\|f\|_\infty^2} \ \text{ given } \ \hDeltajkell \ge b\eta \right\} .
\end{align*}
The first probability can be estimated similarly to how we did for \ref{Delta>hDelta<}.
Thanks to Lemma \ref{lem:rhocsigma}\ref{rho-hrho'}, the second probability is bounded by
$$ \PP\left\{ \rho(\Cjk) < \frac{b^2  \eta^2}{2\|f\|_\infty^2}
\ \text{ and }
\ |\hrho(\Cjk)-\rho(\Cjk)| > \frac{b^2  \eta^2}{2\|f\|_\infty^2}
\right\}
\lesssim \exp\left(-\tfrac{b^2  n\eta^2}{c\|f\|_\infty^2 }\right) $$
for an absolute constant $c$.
Finally, the third probability is zero thanks to \eqref{eq:hDelta<Msqrt(hrho)}.
\end{proof}

\end{appendices}


\end{document}